%% file: root.tex
\documentclass[lettersize,journal]{IEEEtran}

\IEEEoverridecommandlockouts                              
\usepackage{amsmath,amsfonts}

\usepackage{xcolor}
\usepackage{amsthm}
\usepackage{mathtools}
\usepackage{amsmath}
\usepackage{amssymb}
\usepackage{booktabs}
\usepackage{array}
\usepackage{makecell}
\usepackage[caption=false,font=normalsize,labelfont=sf,textfont=sf]{subfig}
\usepackage{textcomp}
\usepackage{stfloats}
\usepackage{url}
\usepackage{verbatim}
\usepackage{graphicx}
\def\BibTeX{{\rm B\kern-.05em{\sc i\kern-.025em b}\kern-.08em
    T\kern-.1667em\lower.7ex\hbox{E}\kern-.125emX}}
\usepackage{balance}
\usepackage{mysymbol}
\usepackage{cite}
\usepackage{hyperref}
\usepackage{dsfont}
\usepackage{algorithm}

\usepackage[noend]{algpseudocode}
\newtheorem{theorem}{Theorem}[section]

\newtheorem{prop}[theorem]{Proposition}

\newtheorem{problem}{Problem}

\newtheorem{rem}[theorem]{Remark}
\newtheorem{ex}[theorem]{Example}

\begin{document}

\title{ConformalNL2LTL: Translating Natural Language Instructions into Temporal Logic Formulas with Conformal Correctness Guarantees}
\author{David Smith Sundarsingh$^*$, Jun Wang$^*$, Jyotirmoy V. Deshmukh, Yiannis Kantaros
\thanks{$^*$ indicates equal contribution.}
\thanks{D. S. Sundarsingh, J. Wang, and Y. Kantaros are with Department of Electrical and Systems Engineering, Washington University in St Louis, MO, 63108, USA. Emails: $\left\{\text{d.s.sundarsingh, junw, ioannisk}\right\}$@wustl.edu. J. Deshmukh is with Department of Computer Science, University of Southern California, Los Angeles, CA 90089, USA. Email: $\text{jdeshmuk}$@usc.edu.}
\thanks{D.S.Sundarsingh, J. Wang, and Y. Kantaros were partially supported by the ARL grant DCIST CRA W911NF-17-2-0181 and the NSF award CNS \#2231257. J. Deshmukh was partially supported through the following NSF grants: CAREER award (SHF-2048094), CNS-2039087, SLES-2417075, finding by Toyota R$\&$D through the USC Center for Autonomy and AI, and the Airbus Institute for Engineering Research. 
}
\thanks{$^1$Additional materials, including videos and code, are available on the project webpage at: \textbf{conformalnl2ltl.github.io}.}
}

\maketitle

\begin{abstract}
Linear Temporal Logic (LTL) is a widely used task specification language for autonomous systems. To mitigate the significant manual effort and expertise required to define LTL-encoded tasks, several methods have been proposed for translating Natural Language (NL) instructions into LTL formulas, which, however, lack correctness guarantees.
To address this, we propose a new NL-to-LTL translation method, called ConformalNL2LTL that achieves user-defined translation success rates on unseen NL commands. Our method constructs LTL formulas iteratively by solving a sequence of open-vocabulary question-answering (QA) problems using large language models (LLMs). These QA tasks are handled collaboratively by a primary and an auxiliary model. The primary model answers each QA instance while quantifying uncertainty via conformal prediction; when it is insufficiently certain according to user-defined confidence thresholds, it requests assistance from the auxiliary model and, if necessary, from the user.
We demonstrate theoretically and empirically that ConformalNL2LTL achieves the desired translation accuracy while minimizing user intervention.
\end{abstract}

\begin{IEEEkeywords}
Formal Methods in Robotics and Automation, Autonomous Agents, AI-Enabled Robotics 
\end{IEEEkeywords}


\vspace{-0.2cm}
\section{Introduction}
\vspace{-0.1cm}
\input{intro_new}

\section{Problem Statement}\label{sec:problem}
\input{problem_new2}

\section{Translating Natural Language to Linear Temporal Logic with Correctness Guarantees}\label{sec:method}
\input{method_gadget_v2}

\section{Experimental Validation}\label{sec:sims}
\input{experiment_v2}

\section{Conclusions and Future Work}
\input{conclusion}

\bibliographystyle{IEEEtran}
\bibliography{YK_bib.bib}

\input{appendix_CP.tex}
\begin{IEEEbiography}
{David Smith Sundarsingh}(S'24) is a PhD student in the Department of Electrical and Systems Engineering, Washington University in St. Louis, MO, USA. He received his B.E. in Mechanical Engineering from Coimbatore Institute of Technology, TN, India in 2020, and his M.Tech degree in Robotics and Automation from the Defence Institute of Advanced Technology, MH, India. His research interests include robotics, machine learning, and formal methods.
\end{IEEEbiography}

\vspace{-1cm}
\begin{IEEEbiography}
{Jun Wang}(S'22) is a PhD candidate in the Department of Electrical and Systems Engineering at Washington University in St. Louis. He received his B.Eng. degree in Software Engineering from Sun Yat-Sen University in 2019 and his MSE degree in Robotics from the University of Pennsylvania in 2021. His research interests include robotics, machine learning, and control theory.
\end{IEEEbiography}

\vspace{-1cm}
\begin{IEEEbiography}
    {Jyotirmoy V. Deshmukh} is an Associate Professor of Computer Science in the Thomas Lord Department of Computer Science and the co-director of the USC center for Autonomy and AI. Previously he was a Principal Research Engineer at Toyota motors R \& D. He was a postdoctoral research scholar at the University of Pennsylvania and received his Ph.D from the University of Texas at Austin.
\end{IEEEbiography}

\vspace{-1cm}
\begin{IEEEbiography}
{Yiannis Kantaros}(S'14-M'18) is an Assistant Professor in the Department of Electrical and Systems Engineering, Washington University in St. Louis (WashU), St. Louis, MO, USA. He received the Diploma in Electrical and Computer Engineering in 2012 from the University of Patras, Patras, Greece. He also received the M.Sc. and the Ph.D.  degrees in mechanical engineering from Duke University, Durham, NC, in 2017 and 2018, respectively. Prior to joining WashU, he was a postdoctoral associate in the Department of Computer and Information Science, University of Pennsylvania, Philadelphia, PA. His current research interests include machine learning, distributed control and optimization, and formal methods with applications in robotics. He received the Best Student Paper Award at the 2nd IEEE Global Conference on Signal and Information Processing (GlobalSIP) in 2014 and was a finalist for the Best Multi-Robot Systems Paper at the IEEE International Conference on Robotics and Automation (ICRA) in 2024 and the Best Paper Award at the ACM/IEEE International Conference on Cyber-Physical Systems (ICCPS) in 2025. He also received the 2017-18 Outstanding Dissertation Research Award from the Department of Mechanical Engineering and Materials Science, Duke University, and a 2024 NSF CAREER Award.
\end{IEEEbiography}

\end{document}

%% file: intro_new.tex


\IEEEPARstart{L}{inear} Temporal Logic (LTL) has become a powerful formalism for specifying a wide range of robotic missions beyond simple reach-avoid requirements, such as surveillance, coverage, and data-gathering \cite{baier2008principles}.
Several task and motion planners have been developed for robots with LTL-encoded tasks, offering correctness and optimality guarantees. These planners have demonstrated success in handling known environments \cite{vasile2013sampling, tumova2016multi, kantaros2020stylus,liu2024time,luo2021abstraction,gujarathi2022mt,fang2024continuous}, unknown static environments \cite{guo2015multi,livingston2013patching,Kantaros2022perception}, and dynamic or uncertain environments \cite{Kalluraya2023multi,articleT-ASE,zhou2023vision,10571830, 10937075}, including scenarios where system dynamics are unknown
\cite{kantaros2024sample,balakrishnan2023model,10990233,hasanbeig2019reinforcement} or robot capabilities may fail unexpectedly \cite{Zhou2022Reactive,kalluraya2023resilient,Feifei2022failure,Faruq2018Simultaneous,10955702, 7778995}.

Despite the remarkable performance of LTL-based planners, a key limitation is the significant manual effort and domain expertise required to specify missions as temporal logic formulas. As a result, natural language (NL) has emerged as an appealing alternative for mission specification. 
Several algorithms have been proposed that leverage Large Language Models (LLMs) to convert NL instructions into temporal logic formulas \cite{chen2024autotamp, fuggitti2023nl2ltl, chen2023nl2tl, liu2023grounding, cosler2023nl2spec,xu2024learning,pan2023data,mohammadinejad2024systematic}, which then serve as task specification inputs to the above-mentioned temporal logic planners.
However, these algorithms lack correctness guarantees, meaning the resulting LTL formula may not accurately reflect the original NL task. This misalignment can result in robot plans that fail to meet the original NL instructions.

To address this limitation, we propose a new translation method, ConformalNL2LTL (Conformal translation of NL commands to LTL formulas), which achieves user-defined translation success rates on previously unseen NL instructions. We consider a robot with a known skill set (e.g., manipulation, sensing, and mobility) tasked with executing high-level NL missions in environments containing semantic objects and regions. Each NL task specifies how the robot should apply its skills to particular objects or regions in a prescribed temporal and logical order. ConformalNL2LTL translates NL commands into LTL formulas defined over atomic propositions (APs), i.e., Boolean variables, that are grounded in the robot’s capabilities and the environment.
%
%
\begin{figure}[t] 
\centering
\includegraphics[width=\linewidth]{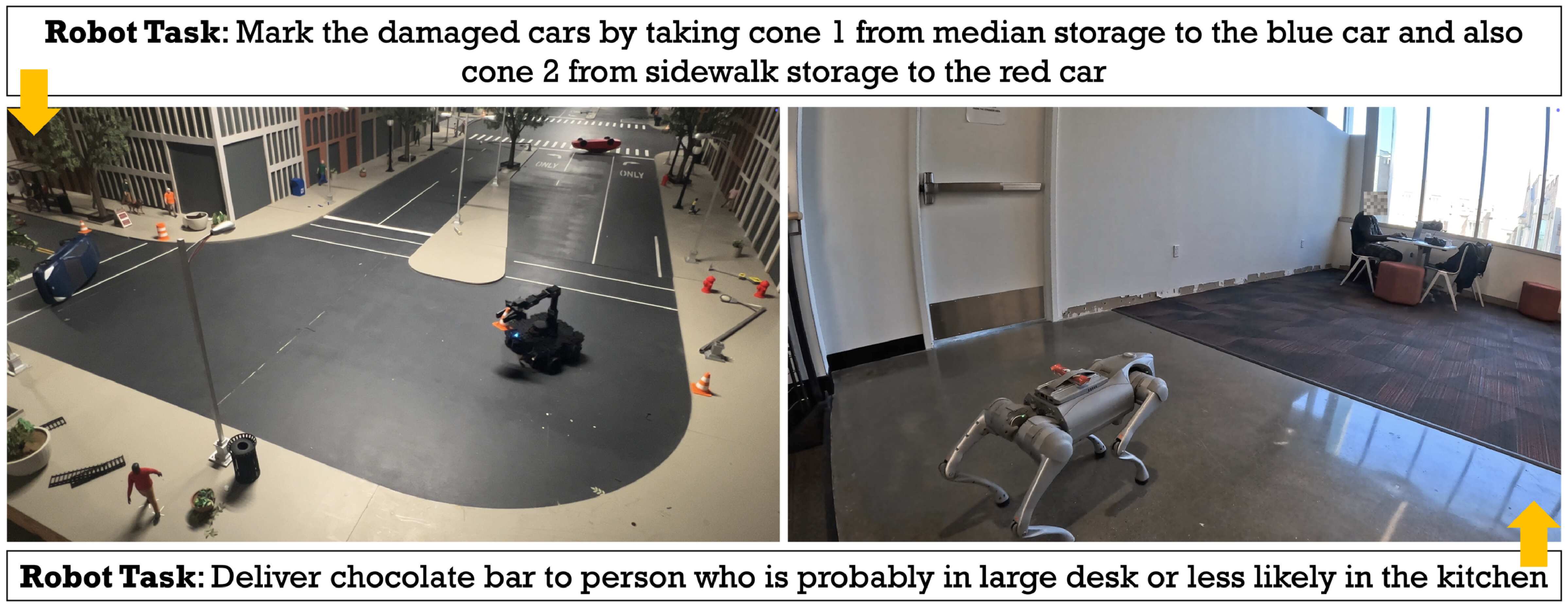} \vspace{-0.8cm}
\caption{ConformalNL2LTL translates natural language-based robot tasks into Linear Temporal Logic formulas with translation correctness guarantees; {see demonstrations in 
\cite{video_demo}.
}}
\label{fig:env}\vspace{-0.7cm}
\end{figure}
Our approach uses pre-trained LLMs as backbone translators, augmented with Conformal Prediction (CP) \cite{shafer2008tutorial,angelopoulos2021gentle}, an uncertainty-quantification framework that plays a central role in achieving user-specified translation success rates. The key idea is to construct the LTL formula iteratively by framing the translation problem as a sequence of interdependent question–answering (QA) tasks, whose answers collectively determine the final specification. These QA tasks are solved collaboratively by two LLMs, referred to as the primary and auxiliary models, with CP coordinating their interaction.
Each QA step incorporates the NL instruction together with the partial LTL formula constructed so far, and prompts the primary LLM to propose either a temporal/logical operator or an AP to extend the specification. A naïve approach would be to query the primary model once and accept its response; however, this ignores the model’s confidence in the correctness of its response. To enable confidence-aware construction of LTL formulas, inspired by \cite{su2024api,cheng2024efficienteqa,kaur2024addressing}, we sample multiple responses per question and use their empirical frequencies as a coarse proxy for model confidence. Rather than selecting the response with the highest empirical confidence (as in, e.g., \cite{ahn2022can,cheng2024efficienteqa}), we employ CP to calibrate this uncertainty. This step is crucial, as LLMs often produce incorrect outputs with high confidence \cite{ren2023robots}, which can result in erroneous translations. Using CP and the responses of the primary model, we construct for each QA task a prediction set that contains the correct operator or AP with a user-specified confidence. When this prediction set is a singleton, the algorithm proceeds with that choice. When the set is non-singleton—indicating high uncertainty—the primary model seeks assistance. Help is first requested from the auxiliary LLM, which is queried in the same manner to produce its own prediction set. We then intersect the prediction sets of the two models. If the intersection is a singleton, the corresponding response is selected; if it remains non-singleton, the user is asked to choose among the remaining options. If the intersection is empty or does not contain a correct response, the user terminates the translation process.

We provide theoretical guarantees showing that ConformalNL2LTL achieves user-specified translation success rates under the standard distributional assumptions required for CP. Extensive numerical experiments further demonstrate that our method attains high success levels (e.g., 99\%) on unseen NL instructions while requiring very low rates of user intervention (e.g., 0.378\%). We also show that ConformalNL2LTL outperforms existing translation methods \cite{liu2023grounding,cosler2023nl2spec}, which do not explicitly account for uncertainty, and that it continues to outperform these baselines even when the CP-based help mechanism is disabled and only the highest-confidence responses from the primary model are used. For completeness, we compare an integration of ConformalNL2LTL with the TL-RRT* planner \cite{luo2021abstraction} against methods that map NL instructions directly to robot plans, bypassing 
the translation to LTL formulas
\cite{chen2023scalable}, and observe similar conclusions. We further demonstrate the planning capabilities of the integrated framework for language-instructed robots; see Fig.~\ref{fig:env}. Finally, we empirically demonstrate strong performance of ConformalNL2LTL under distribution shifts that violate CP assumptions and compromise our theoretical guarantees.

\textbf{Related Works}: 
(i) \textit{NL-to-LTL Translation:} As discussed earlier, several NL-to-LTL translation methods have been proposed, generating LTL formulas that are then used as inputs to existing planners \cite{chen2024autotamp,fuggitti2023nl2ltl,chen2023nl2tl,liu2023grounding,cosler2023nl2spec,xu2024learning,pan2023data,mohammadinejad2024systematic}. However, unlike ConformalNL2LTL, these approaches are uncertainty-agnostic: they treat LLM outputs as ground truth, ignoring inherent model uncertainty and confidence, and consequently provide no guarantees on translation correctness.

\textit{(ii) NL-based Planners:} Related to our work are NL-based planners that bypass the construction of LTL formulas altogether. These approaches use LLMs or Vision-Language Models to map NL instructions directly to control or action sequences \cite{hunt2025surveylanguagebasedcommunicationrobotics,Pallagani_Muppasani_Roy_Fabiano_Loreggia_Murugesan_Srivastava_Rossi_Horesh_Sheth_2024,chen2023scalable,9197464,latif20243pllmprobabilisticpathplanning,10160591,zhang2023building,yang2024text2reaction,hong2023metagpt,shah2023navigationlargelanguagemodels}. 
While such methods demonstrate strong empirical performance, they similarly lack guarantees on task success or mission-level correctness. Moreover, their end-to-end nature offers limited interpretability. In contrast, translation-based frameworks, including ours, can be modularly composed with existing planning and control stacks \cite{vasile2013sampling, tumova2016multi, kantaros2020stylus,liu2024time,luo2021abstraction,gujarathi2022mt,fang2024continuous,guo2015multi,livingston2013patching,Kantaros2022perception,Kalluraya2023multi,zhou2023vision,kantaros2024sample,balakrishnan2023model,hasanbeig2019reinforcement,wang2023task,Zhou2022Reactive,kalluraya2023resilient,Feifei2022failure,Faruq2018Simultaneous}, providing traceability and allowing users to localize errors to specific stages (e.g., translation vs. planning).

(iii) \textit{CP for Verified Autonomy:} CP has been widely used for uncertainty quantification in autonomy tasks, including object tracking \cite{su2024collaborative}, perception \cite{dixit2024perceive}, trajectory prediction \cite{lindemann2023safe}, reachability analysis \cite{hashemi2023data}, and decision-making \cite{sun2023conformal,huriot2025safe}. More recently, CP has also been used to quantify uncertainty in LLM-based planners \cite{ren2023robots,heracles,wang2024safe,ren2024explore,liang2024introspective}. However, these works are limited to open-source LLMs where logit/confidence scores are available. CP for closed-source LLMs has been explored in \cite{su2024api,kaur2024addressing}, but only for single QA tasks. Our work extends these efforts by enabling CP-based uncertainty quantification for both open- and closed-source LLMs and by handling a sequence of interdependent QA tasks that are necessary for constructing LTL formulas. Additionally, these existing CP-enabled LLM-based approaches rely solely on a user to resolve high-uncertainty cases, resulting in frequent human interventions, especially as the desired success probability increases, thus limiting autonomous deployment \cite{ren2023robots,heracles,wang2024safe,ren2024explore,liang2024introspective}. We address this limitation by invoking an auxiliary LLM to resolve high-uncertainty cases before involving the user, yielding a significant reduction (e.g., from $36.5\%$ to $4\%$) in the proportion of formulas that required help in their construction from the user at least once, as shown empirically (see Section \ref{sec:auxModelHelp}). Finally, to the best of our knowledge, this work is the first to apply CP to LLM-based NL-to-LTL translation.

\textbf{Contribution}: The contribution of this paper can be summarized as follows. \textit{First}, we propose ConformalNL2LTL, the first NL-to-LTL translation algorithm capable of achieving user-specified translation success rates on unseen NL instructions. 
\textit{Second}, we present the first application of CP to LLMs working collaboratively for NL-to-LTL translation tasks, with the added benefit that our method can be applied to both open- and closed-source LLMs.
\textit{Third}, we provide both theoretical and empirical analyses demonstrating that ConformalNL2LTL can achieve user-specified translation success rates while maintaining low help rates from users. \textit{Fourth}, we release the software implementation of ConformalNL2LTL, along with datasets developed for calibration and testing purposes \cite{code_demo} 

%% file: problem_new2.tex

\textbf{Robot and Environment Modeling:} Consider a robot equipped with a \textit{known} set of skills/actions, denoted as $\mathcal{A} = \{a_1, \dots, a_A\}$ (e.g., `take a picture', `grab', or `go to') operating in an environment $\Omega=\{l_1,\ldots,l_M\}$ that contains $M > 0$ semantic objects and regions (called, hereafter, landmarks) $l_j$. Each landmark $l_j$ is determined by its position and its semantic label (e.g., `box', `car', `kitchen'). The robot's skills can be applied to $l_j$ (e.g., `go to $l_j$' or `grab $l_j$').

\textbf{Natural Language Tasks:} The robot is assigned a task $\xi$, expressed in natural language (NL), which may comprise multiple sub-tasks. The task $\xi$ is grounded in the robot skill set $\ccalA$ and a subset of the environment $\Omega$, meaning it can be accomplished by executing a sequence of actions selected from $\ccalA$ on some of the landmarks within $\Omega$; see Ex. \ref{ex1}. To formalize this, we consider a distribution $\ccalD$ over translation scenarios $\sigma_i=\{\xi_i, \ccalA_i\}$, where $\xi_i$ is defined over an unknown environment $\Omega_i$; see also Remark \ref{rem:Distr}. The subscript $i$ is used to emphasize that the NL-encoded task and the robot skills can vary across scenarios. When it is clear from the context, we drop the dependence on $i$. Note that $\ccalD$ is unknown but we assume that we can sample i.i.d. scenarios from it.  

\textbf{Linear Temporal Logic:}  LTL is a formal language that comprises a set of atomic propositions (APs) $\pi$ (i.e., Boolean variables), collected in a set denoted by $\mathcal{AP}$, Boolean operators, (i.e., conjunction $\wedge$, and negation $\neg$), and temporal operators, such as \textit{always} $\square$, \textit{eventually} $\lozenge$, and until $\mathcal{U}$.  
{We focus on LTL-encoded robot tasks defined over APs grounded in the robot's skill set $\ccalA_i$ and environment $\Omega_i$. Specifically, we consider APs that are true when a skill $a_i$ (e.g., `pick up') is applied to a landmark $l_j$ (e.g., `package'); see also Remark \ref{rem:Distr}. These LTL formulas $\phi$ are satisfied by robot plans $\tau$, which are sequences of robot states and corresponding actions. A formal presentation of the syntax and semantics of LTL, as well as robot plans satisfying LTL tasks can be found in \cite{baier2008principles,kalluraya2023resilient}; designing plans $\tau$ is out of the scope of this work.}
\textbf{Problem Statement:}
Our goal is to design a probabilistically correct translation algorithm that converts NL instructions $\xi_i$, contained in translation scenarios $\sigma_i = \{\xi_i, \mathcal{A}_i\}$ sampled from a distribution $\mathcal{D}$, into semantically equivalent LTL-encoded tasks $\phi_i$, with probability at least $(1-\alpha)$ over scenarios drawn from $\mathcal{D}$, for a user-specified $\alpha \in (0,1)$. In other words, the desired translation success rate over $\mathcal{D}$ is at least $(1-\alpha)\cdot 100$\%. Given a scenario $\sigma_i = \{\xi_i, \mathcal{A}_i\}$, we say that an LTL formula $\phi_i$ is semantically equivalent to $\xi_i$ if any plan $\tau$ satisfying $\phi_i$ also satisfies $\xi_i$, and vice versa. When multiple LTL formulas are semantically equivalent to $\xi_i$, our goal is to compute one such formula.
%
%
%
The problem that the paper addresses can be formally stated as follows:
\begin{problem}\label{prob:main}
{Design a translation algorithm that, given a translation scenario $\sigma=\{\xi,\ccalA\}$ sampled from an unknown distribution $\mathcal{D}$, generates an LTL formula $\phi$ such that 
\begin{align}
\mathbb{P}_{\sigma\sim\ccalD}(\phi\equiv\xi)\geq 1-\alpha,
\end{align}
for a user-specified threshold $1-\alpha\in(0,1)$ (where $\equiv$ stands for semantic equivalence).
}
\end{problem}

\begin{ex}[NL-to-LTL]\label{ex1}
    Consider an example of a mobile manipulator with action space $\mathcal{A}=$ \{move to, pick up, put down\}, and NL task $\xi =$ `Pick up the red box and place it in storage'. 
    Our proposed algorithm will infer that the corresponding LTL formula should be defined over three APs: {(i) $p\_red\_box$ which is true if the `red box' was `picked up' by the robot; 
    (ii) $storage$ which is true if the robot is in the `storage' area; and (iii) $pd$ which is true if the carried object is placed.}
    %
    This will result in the following LTL formula $\phi = \lozenge(p\_red\_box\wedge\lozenge(storage\wedge pd))$
    which is satisfied when $p\_red\_box$ becomes true first, followed by satisfaction of both $storage$ and $pd$. 
\end{ex}

\begin{rem}[Distribution $\ccalD$]\label{rem:Distr}
Note that $\ccalD$ is defined over $\sigma$, and not on specific environment instantiations $\Omega$. 
Instead, the randomness of $\Omega$ is implicitly accounted for when sampling the task $\xi$. This means that $\Omega$ may vary across NL tasks $\xi_i$ sampled from $\ccalD$. Similarly, $\ccalD$ does not depend on a predefined set of APs for constructing $\phi$. 
\textcolor{black}{Instead, the translation algorithm is tasked with inferring the set $\mathcal{AP}$ from $\xi$ and the provided skill set $\ccalA$. In practice, this means that users only need to specify the skill set $\ccalA$ of the platform and the NL instructions $\xi$, making our algorithm environment-agnostic. This enhances suitability for open-world task settings and also improves usability by eliminating the need to manually define APs or environments for each new scenario $\sigma$.}
\end{rem}

%% file: method_gadget_v2.tex
In this section, we present ConformalNL2LTL, our proposed translation framework to address Problem~\ref{prob:main}.
In Section~\ref{sec:Translation}, we provide an overview of the framework and highlight its key ideas.
Then, in Section~\ref{sec:primaryLLM}, we show how the translation problem can be formulated as a sequence of question–answering (QA) tasks, which are solved by a pre-trained LLM in an uncertainty-aware manner via conformal prediction (CP).
In Section~\ref{sec:auxLLM}, we explain how ConformalNL2LTL manages high-uncertainty cases, flagged by CP, by querying an auxiliary LLM and, if needed, deferring to a user.
Finally, in Section~\ref{sec:cp}, we discuss how CP is used to quantify the uncertainty of the LLMs involved in solving these QA tasks.



\begin{figure}[t] 
\centering
\includegraphics[width=\linewidth]{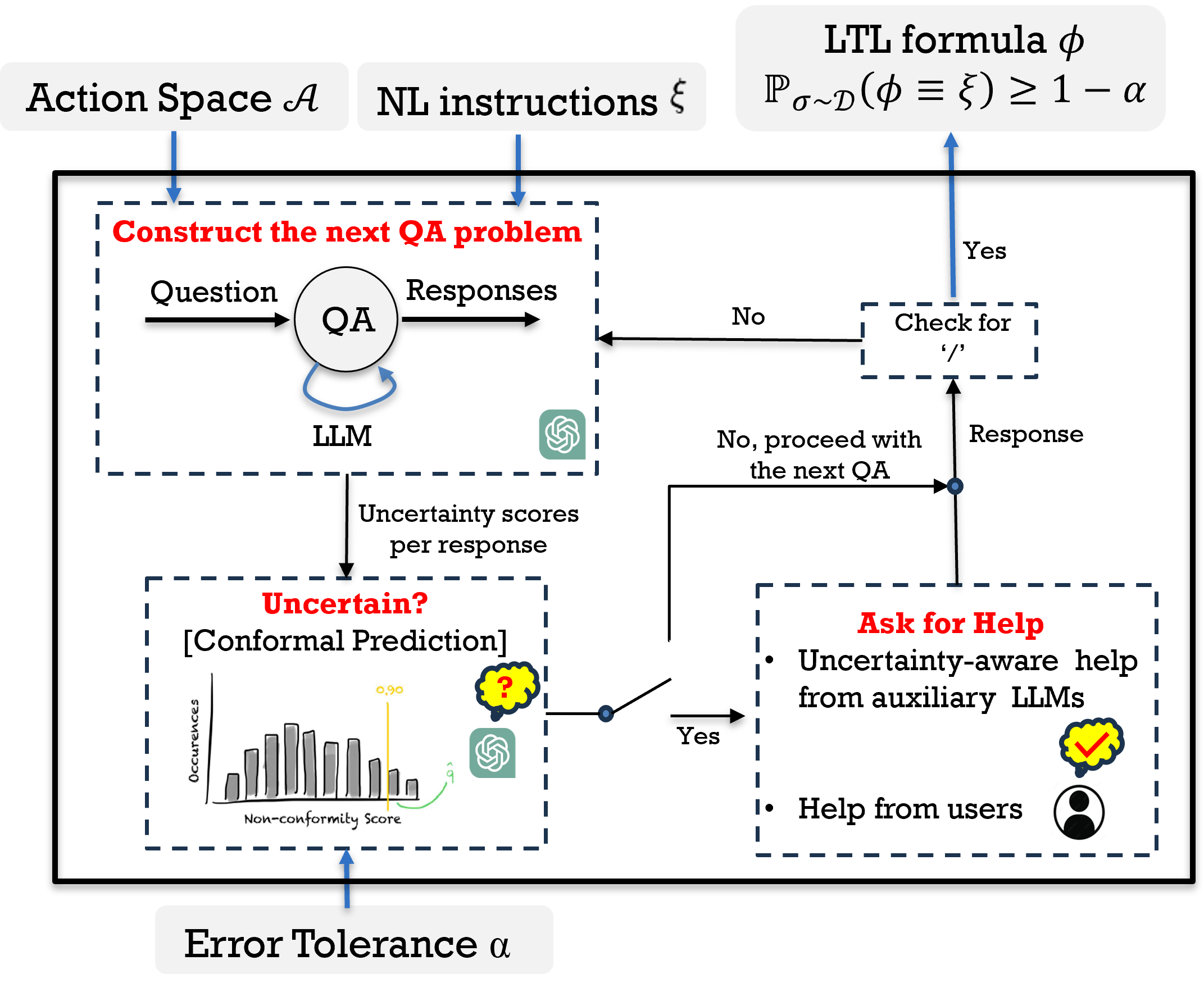}\vspace{-0.6cm}
\caption{Overview of ConformalNL2LTL. Given a translation scenario $\sigma=\{\xi,\ccalA\}$ and a threshold $\alpha$, a sequence of QA problems are  solved by an LLM. At every step, if the LLM is uncertain about its response, it asks for help from an auxiliary LLM, and a user, if required. If the response to a QA problem is '/', then the LTL formula $\phi$ is built by concatenating the responses generated at every QA step.}
\label{fig:overview}
\end{figure}

\begin{algorithm}[t]
\footnotesize
\caption{ConformalNL2LTL}\label{alg:translation}
\begin{algorithmic}[1]
\State \textbf{Input}: NL Task $\xi$; Coverage level $1-\alpha$; Primary LLM $\psi_{\text{p}}$; Auxiliary LLM $\psi_{\text{aux}}$, NCS quantile $\bar{q}$, semantic threshold $\zeta$
\State Initialize $k=1$ and prompt $\ell(1)$ with formula $\phi(0)=\varnothing$\label{algo1:prompt_init}
\While{`$/$'$\not\in\phi(k-1)$}
    \State $\mathcal{C}(\ell(k),{\psi_\text{p}})=\texttt{PredSet}(\ell(k),{\psi_\text{p}},\bar{q},\zeta)$ (See Alg.~\ref{algo3:algorithm})\label{algo1:pred_set}

    \If{$|\mathcal{C}(\ell(k),{\psi_\text{p}})|>1$}\label{algo1:HelpTriggered}
        \State $\mathcal{C}(\ell(k),\psi_{\text{aux}})=\texttt{PredSet}(\ell(k),\psi_{\text{aux}},\bar{q},\zeta)$ (See Alg.~\ref{algo3:algorithm})\label{algo:auxPred}
        \State $\mathcal{C}_{\text{inter}}(\ell(k))=\mathcal{C}(\ell(k),{\psi_\text{p}})\cap\mathcal{C}(\ell(k),\psi_{\text{aux}})$\label{algo:intersection}

        \If{$|\mathcal{C}_{\text{inter}}(\ell(k))|>1$}\label{algo:humanHelp}
            \State {Ask human operator to select $s(k)$ from $\mathcal{C}_{\text{inter}}(\ell(k))$}\label{algo1:HelpHuman2}
        \ElsIf{$\mathcal{C}_{\text{inter}}(\ell(k))=\emptyset$ {\text{or} `correct response $\not\in$ $\mathcal{C}_{\text{inter}}(\ell(k))$'}}
            \State Halt and declare translation as failed\label{algo1:halt}
        \Else
            \State Pick the (unique) decision $s(k)\in\mathcal{C}_{\text{inter}}(\ell(k))$\label{algo:finalPick}
        \EndIf

    \ElsIf{$\mathcal{C}(\ell(k),\psi_{\text{p}})=\emptyset$ 
    }
        \State Halt and declare translation as failed\label{algo1:haltprimary}

    \Else
        \State Pick the (unique) decision $s(k)\in\mathcal{C}(\ell(k),{\psi_\text{p}})$\label{algo1:pickAction}
    \EndIf

    \State Update $\phi(k)=\phi(k-1)+s(k)$ \label{algo1:formula_update}
    \State Construct $\ell(k+1)$ using $\ell(k)$ and $\phi(k)$\label{algo1:constrPrompt}
    \State Update $k=k+1$ \label{algo1:Updatek}
\EndWhile
\State \textbf{Output}: LTL formula $\phi$
\end{algorithmic}
\end{algorithm}
\normalsize

\vspace{-0.3cm}
\subsection{Overview of ConformalNL2LTL}\label{sec:Translation}

\textcolor{black}{In this section, we provide an overview of ConformalNL2LTL summarized also in Alg. \ref{alg:translation} and Figure \ref{fig:overview}.
Consider a translation scenario $\sigma=\{\xi,\ccalA\}$ drawn from $\ccalD$. A key idea of ConformalNL2LTL is to frame the translation process as a sequence of interdependent question-answering (QA) problems, each solved collaboratively by two pre-trained LLMs, referred to here as the \textit{primary} and \textit{auxiliary} LLMs, denoted by $\psi_{\text{p}}$ and $\psi_{\text{aux}}$, respectively. 
The answers to these QA problems collectively form the LTL formula. Specifically, at each iteration $k$, a QA problem is formulated. The \textit{question} incorporates information about the NL task $\xi$, the set of robot skills $\ccalA$, and the responses from previous iterations (i.e., the partially constructed LTL formula). This textual information, denoted by $\ell(k)$, is provided as a prompt to the primary LLM, which is tasked with generating an \textit{answer} $s(k)$ that is either a temporal/logical operator or an AP. If the primary LLM is not sufficiently confident about the correctness of its response, according to the user-defined threshold $1-\alpha$ introduced in Problem \ref{prob:main}, help is requested first from the auxiliary LLM and, if necessary, from a user. This coordination among the LLMs and the user is enabled via CP (applied as shown in Alg. \ref{algo2:algorithm} and \ref{algo3:algorithm}). This iterative process continues until an ending operator, denoted by `$/$', is generated at some step $H$, indicating that the LTL formula $\phi = s(1)s(2)\ldots s(H)$ has been fully constructed. 
In the following subsections, we describe the translation process in detail.}

\begin{figure}[t] 
\centering
\includegraphics[width=\linewidth]{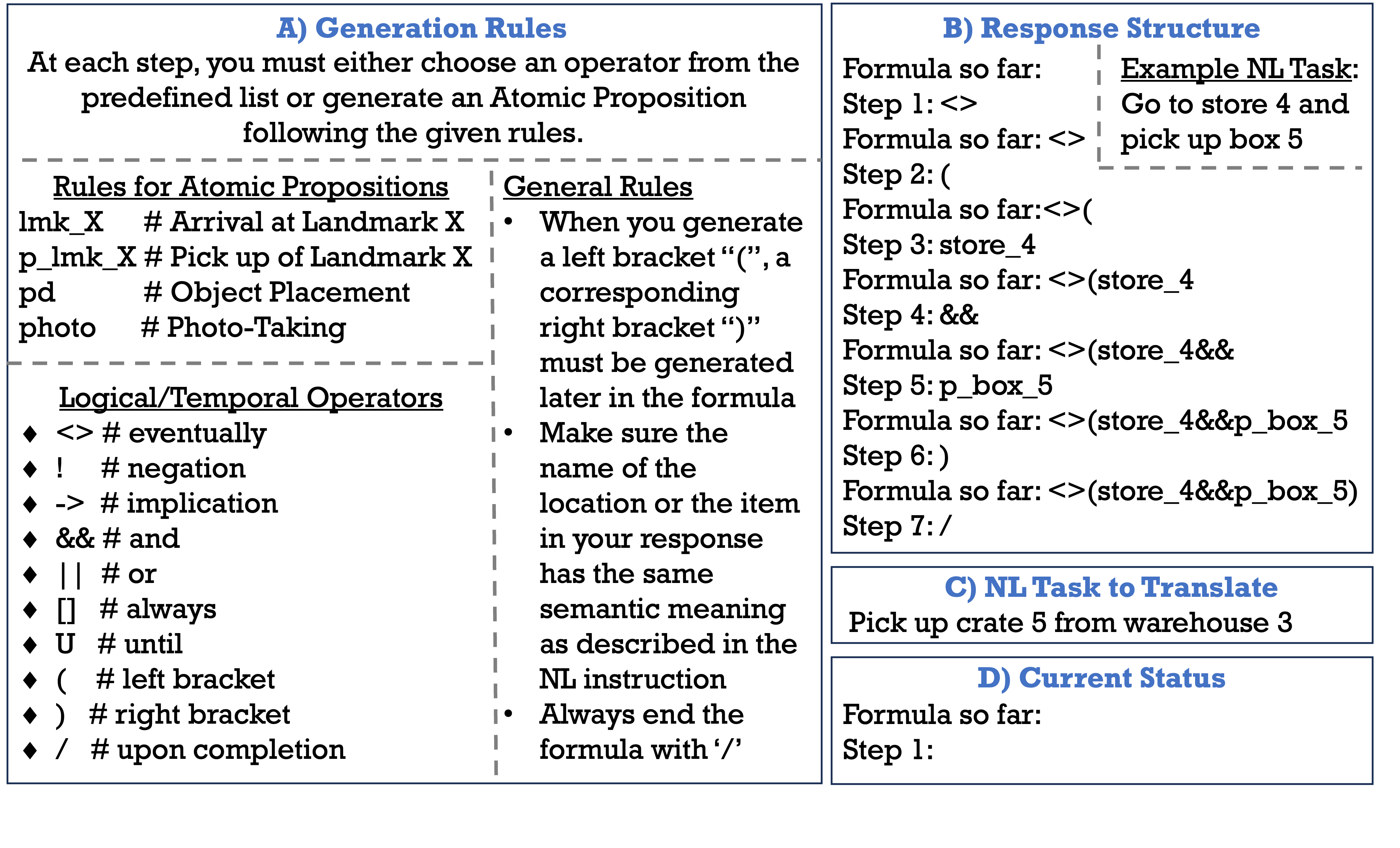}\vspace{-0.9cm}
\caption{\textcolor{black}{Example of the constructed prompt \textcolor{black}{for GPT-4o. } This prompt refers to $k=1$ {meaning that no response has been selected yet and the formula constructed so far is therefore `empty'.}
}} \vspace{-0.4cm}
\label{fig:prompt}
\end{figure}

\subsection{Incremental Translation: Querying the Primary LLM\label{sec:primaryLLM}}

In what follows, we describe the role of the primary LLM. Specifically, we discuss how the prompt $\ell(k)$ is constructed, followed by how it is used to solicit a response $s(k)$ from the primary model in an uncertainty-aware manner using CP.

\textbf{Prompt Structure:} The prompt $\ell(k)$ used for querying the primary LLM at every step $k$ has the following structure: (a) \textit{Translation rules} that defines: (i) the list of all valid temporal and logical operators; (ii) rules for generating APs that are grounded on the robot skill-set $\ccalA$ {and the landmarks mentioned in the task $\xi$}; and (iii) general rules that the LLM should follow during the translation process (e.g., about the ending operator '$/$'). Part (a) prompts the LLM to generate an operator or an AP using the rules (i)-(iii) towards constructing a formula that is semantically equivalent to $\xi$. (b) \textit{Response structure} with {$n$-shot examples describing the desired structure of the LLM output at each step $k$, where $n$ depends on the language model}; (c) \textit{Task description} of an NL task $\xi$ that needs to be converted into an LTL formula; (d) \text{Current status} describing the formula $\phi(k-1)$ that has been built up to step $k-1$ and the current step indicator $k$. An example of a prompt is shown in Fig. \ref{fig:prompt}.

\textbf{{Initial} Prompt Construction:} Given a scenario $\sigma=\{\xi,\mathcal{A}\}$, the initial prompt $\ell(1)$ is manually constructed  with the initial formula in part (d) being empty [line \ref{algo1:prompt_init}, Alg. \ref{alg:translation}].
Parts (a)-(c) in $\ell(k)$ are the same for all $k\geq 1$. {Part (d) is updated at every step $k$ once a response $s(k)$ is generated; this will be discussed later in the text.}
%



\textbf{Generating Confidence-aware Responses:}
Next, we describe how to obtain a response $s(k)$ corresponding to $\ell(k)$. A straightforward approach is to query the primary LLM once and record its response; {we denote this process by $s(k) = \texttt{Prompt}(\psi_{\text{p}}, \ell(k))$}. However, this single-query approach is uncertainty-agnostic, as it does not capture the model’s confidence in its predictions. To address this, inspired by \cite{cheng2024efficienteqa, su2024api}, we adopt a sampling-based approach: the model is queried multiple times using $\texttt{Prompt}(\psi_{\text{p}}, \ell(k))$, and the frequency of each unique response is used as a coarse-grained measure of uncertainty. In what follows, we describe this process in detail; 
see Alg. \ref{algo2:algorithm} and Ex \ref{ex2}.\footnote{Note that the output of Alg. \ref{algo2:algorithm} will be used later in Alg. \ref{algo3:algorithm} to generate prediction sets that are crucial for ConformalNL2LTL.}

\begin{algorithm}[t]
\footnotesize
\caption{$\texttt{GetResp}$}\label{algo2:algorithm}
\begin{algorithmic}[1]
\State \textbf{Input}: LLM $\psi$, prompt $\ell(k)$, $m$, semantic threshold $\zeta$
\State Initialize $\hat{\ccalS}^\psi(k)=\emptyset$
\For{i=$1$ to $m$}\label{algo2:forprompt}
\State $\hat{\ccalS}^\psi(k)=\hat{\ccalS}^\psi(k)\cup\{\texttt{Prompt($\psi$,$\ell(k)$)}\}$\label{algo2:prompt}
\EndFor
\State $\ccalS^\psi(k)=\texttt{Prune}(\hat{\ccalS}^\psi(k))$\label{algo2:prune}
\For{$s(k)\in\ccalS^\psi(k)$}\label{algo2:freqLoop}
\State Calculate $F^\psi(s(k)|\ccalS(k))$ as in \eqref{eq:freqM}\label{algo2:freqCalc}
\EndFor
\For{$s^{(i)}(k),s^{(j)}(k)\in\ccalS^\psi(k)$}\label{algo2:ssstart}
\If{$\texttt{S.S}(s^{(i)}(k),s^{(j)}(k))\geq \textcolor{black}{\zeta}$}\label{algo2:SS}
\State $l=\argmin_{l\in\{i,j\}}F^\psi(s^{(l)}(k))$
\State $o=\{i,j\}\setminus\{l\}$
\State $s^{(l)}(k)\leftarrow s^{(o)}(k)$\label{algo2:ssend}
\EndIf
\EndFor
\State $\Bar{\ccalS}^\psi(k)=\texttt{Unique}(\ccalS^\psi(k))$\label{algo2:unique}
\For{$\Bar{s}(k)\in\Bar{\ccalS}(k)$}
\State Calculate $F^\psi(\Bar{s}(k)|\ccalS(k))$ as in \eqref{eq:freqM}\label{algo:reComputeFreq}
\EndFor
\State \textbf{Output:} $[\Bar{\ccalS}^\psi(k),F^\psi(k)]$
\end{algorithmic}
\end{algorithm} 

\textit{Querying the model and rule-based pruning:}
We first query the model $\psi_{\text{p}}$ $m$ times with the same prompt $\ell(k)$, and construct the set $\hat{\mathcal{S}}^{\psi_{\text{p}}}(k)=\{\hat s^{(i)}(k)\}_{i=1}^m$ [line \ref{algo2:prompt}, Alg. \ref{algo2:algorithm}], 
where $\hat s^{(i)}(k)$ is the $i$-th response of the LLM, i.e., $\hat s^{(i)}(k)=\texttt{Prompt}(\psi_{\text{p}},\ell(k))$.
To ensure adherence to the constraints specified in part (a) of the prompt $\ell(k)$ (see Fig. \ref{fig:prompt}), out of all the $m$ responses, we filter out any responses that: (i) include an invalid operator that is not part of the predefined set; (ii) APs that violate the rules set up in (a); and (iii) contains multiple operators or APs [line \ref{algo2:prune}, Alg. \ref{algo2:algorithm}]. This yields a set ${\mathcal{S}^{\psi_{\text{p}}}}(k)\subseteq\hat{\mathcal{S}}^{\psi_{\text{p}}}(k)$, such that ${\mathcal{S}}^{\psi_\text{p}}(k)=\{ s^{(i)}(k)\}_{i=1}^{m_k}$, {where $m_k$ denotes the number of valid responses}. 
This process is denoted by  $\ccalS^{\psi_\text{p}}(k)=\texttt{Prune}(\hat{\ccalS}^{\psi_\text{p}}(k))$.

\textit{Frequency as a measure of uncertainty:}
We compute the frequency of each distinct response $s\in{\mathcal{S}}^{\psi_{\text{p}}}(k)$ as follows [lines \ref{algo2:freqLoop}-\ref{algo2:freqCalc}, Alg. \ref{algo2:algorithm}]: \begin{align}\label{eq:freqM}
    F(s|\ell(k),m_k,\mathcal{S}^{\psi_{\text{p}}}(k))=\frac{1}{m_k}\sum_{s^{(i)}(k)\in{\mathcal{S}}^{\psi_{\text{p}}}(k)}\mathds{1}(s^{(i)}(k)=s),
\end{align}
\noindent where $F(s \mid \ell(k), m_k, \mathcal{S}^{\psi_{\text{p}}}(k)) \in (0,1]$ represents the fraction of valid responses equal to $s$, i.e., the number of times $s$ was generated divided by the total number of valid responses $m_k>0$.
Hereafter, for simplicity, when it is clear from the context, we abbreviate $F(s|\ell(k), m_k,\mathcal{S}^{\psi_{\text{p}}}(k))$ as $F^{\psi_{\text{p}}}(s|\ccalS(k))$. Observe that every response in $\ccalS^{\psi_{\text{p}}}(k)$ has a positive frequency. For any other response $s$ that is not among the $m_k$ generated responses (and thus does not belong to 
$\ccalS^{\psi_{\text{p}}}(k)$),  
we set $F^{\psi_{\text{p}}}(s|\ccalS(k))=0$. {Note that $F^{\psi_{\text{p}}}(s|\ccalS(k))=0$ for any response when all $m$ responses are invalid, i.e. $\ccalS^{\psi_{\text{p}}}(k)=\emptyset$ (or, equivalently, $m_k=0$).} We use the frequency $F^{\psi_{\text{p}}}(s|\ccalS(k))$ as the proxy for the  confidence of the model $\psi_{\text{p}}$ on each response $s\in\mathcal{S}^{\psi_{\text{p}}}(k)$ 
since higher response frequency correlates with greater model confidence \cite{wang2022self,su2024api}. 
This enables our method to work seamlessly with  LLMs that do not allow users to have access to logits modeling confidence (e.g., GPT 4) \cite{cheng2024efficienteqa}. {We emphasize that more fine-grained measures of uncertainty can also be employed, as explored in \cite{su2024api,kaur2024addressing}.}

\textit{Merging Semantically Similar Responses: }It is possible for two textually different responses, $s^{(i)}(k)$ and $s^{(j)}(k)$, to convey the same semantic meaning. This issue arises particularly in responses involving APs, which as discussed in Section \ref{sec:problem}, capture the application of a robot action on a landmark, where multiple synonymous terms (e.g., `vehicle' and `car') may be used for the same landmark. To address this, inspired by \cite{cheng2024efficienteqa}, we refine the response frequencies by grouping semantically similar responses and treating them as a single entity; this essentially results in summing their frequencies. 

Formally, consider any two responses $s^{(i)}(k),s^{(j)}(k)\in{\mathcal{S}^{{\psi_\text{p}}}}(k)$ {that are both APs}. 
If their semantic similarity, computed by a function $\texttt{S.S}$ [line \ref{algo2:SS}, Alg. \ref{algo2:algorithm}], exceeds a predefined threshold $\zeta$, we consider the responses semantically equivalent. {Semantic similarity can be computed, e.g., using cosine similarity; our implementation of $\texttt{S.S}$ is discussed in Section~\ref{sec:sims}.}
%
{When two APs are deemed semantically equivalent,} then we retain the more frequent response-i.e., if $F^{{\psi_\text{p}}}(s^{(i)}(k)|\ccalS(k))\geq F^{{\psi_\text{p}}}(s^{(j)}(k)|\ccalS(k))$, we replace $s^{(j)}(k)$  with $s^{(i)}(k)$ in the set $\mathcal{S}^{{\psi_\text{p}}}(k)$. This process is repeated for all distinct pairs $s^{(i)}(k),s^{(j)}(k)\in{\mathcal{S}^{{\psi_\text{p}}}}(k)$, {that are both APs,} to ensure that every remaining response in the set is semantically distinct [lines \ref{algo2:ssstart}-\ref{algo2:ssend}, Alg. \ref{algo2:algorithm}]. 

{We define $\Bar{\mathcal{S}}^{{\psi_\text{p}}}(k) \subseteq \mathcal{S}^{{\psi_\text{p}}}(k)$ as the set of unique responses in $\mathcal{S}^{{\psi_\text{p}}}(k)$ using the $\texttt{Unique}$ function [line \ref{algo2:unique}, Alg. \ref{algo2:algorithm}].}
For each response ${\bar{s}} \in \Bar{\mathcal{S}}^{{\psi_\text{p}}}(k)$, we {re-compute the scores $F(\Bar{s}|\ell(k),m_k,\ccalS^{\psi_p}(k))$, using \eqref{eq:freqM}, to account for the above-mentioned replacements of responses in $\ccalS^{\psi_p}(k)$ due to semantic similarities} [line \ref{algo:reComputeFreq}, Alg. \ref{algo2:algorithm}]. {Hereafter, for brevity, we denote the above response extraction process with frequency scores as $[\Bar{\mathcal{S}}^{{\psi_\text{p}}},F^{\psi_\text{p}}]=\texttt{GetResp}(\psi_\text{p},\ell(k),m,\zeta)$; see Ex \ref{ex2}.}


\textbf{Selecting a Response:} Using these scores modeling the confidence of  ${\psi_\text{p}}$, a straightforward approach to selecting the response $s(k)$, for a given prompt $\ell(k)$ is choosing the response $\bar{s} \in \Bar{\mathcal{S}}^{{\psi_\text{p}}}(k)$ with the highest confidence score, i.e., 
\begin{equation}
    s(k)= \arg \max_{\bar s\in\Bar{\mathcal{S}}^{{\psi_\text{p}}}(k)} F(\bar s|\ell(k), m_k,\mathcal{S}^{{\psi_\text{p}}}(k)).
\end{equation}
However, these scores do not represent the calibrated confidence \cite{ren2023robots}. Thus, a much-preferred approach would be to generate a set of responses (called, hereafter, \textit{prediction set}), denoted by $\ccalC(\ell(k),{\psi_\text{p}})$ that contains the ground truth response with a user-specified confidence (that depends on the threshold $1-\alpha$ introduced in Problem \ref{prob:main}) [line \ref{algo1:pred_set}, Alg. \ref{alg:translation}]. In what follows, we assume the prediction sets are provided. We defer their construction to Section \ref{sec:cp}, but we emphasize that they are critical to achieving the desired $1-\alpha$ translation accuracy.

Given $\ccalC(\ell(k),{\psi_\text{p}})$, we select response $s(k)$ as follows. If $\ccalC(\ell(k),{\psi_\text{p}})$ is a singleton (i.e., $|\ccalC(\ell(k),{\psi_\text{p}})|=1$), then we select the response i.e., a temporal/logical operator or an AP, included in $\ccalC(\ell(k),{\psi_\text{p}})$ as it contributes to the translation process with high confidence [line \ref{algo1:pickAction}, Alg. \ref{alg:translation}]. If however, the prediction set comes out empty due to no valid responses from $\psi_{\text{p}}$, the translation process is terminated [line \ref{algo1:haltprimary}, Alg, \ref{alg:translation}]. {Otherwise, if $|\ccalC(\ell(k),{\psi_\text{p}})|>1$ the primary model seeks helps from the auxiliary model $\psi_{\text{aux}}$ to select $s(k)$. If the primary and the auxiliary model cannot jointly provide a response $s(k)$, then help from a user is requested who may either select $s(k)$ from a set of responses or may terminate the translation process in case that set does not contain the correct response. This help module, is presented in the Section \ref{sec:auxLLM}.}


{\textbf{Updating the Formula and the Prompt:}} Once $s(k)$ is selected {by the primary model or through the help module}, the formula is updated as $\phi(k) = \phi(k-1) + s(k)$, 
where with slight abuse of notation, the `+' denotes concatenation of texts [line \ref{algo1:formula_update}, Alg. \ref{alg:translation}]. In this iterative update, {$\phi(0)$ is an `empty' formula used in the construction of the initial prompt $\ell(1)$.}
We then update the iteration index [line \ref{algo1:Updatek}, Alg. \ref{alg:translation}], for which the prompt, $\ell(k+1)$, is constructed by replacing part (d) of $\ell(k)$ with $\phi(k)$ [line \ref{algo1:constrPrompt}, Alg. \ref{alg:translation}]; see also Fig. \ref{fig:prompt}. This procedure is repeated sequentially until, at some time step $k=H$, the ending operator `$/$' is generated, indicating that the formula has been fully constructed as $\phi = s(1) \dots s(H)$. In Section \ref{sec:theory}, we show that this algorithm achieves user-defined translation success rates $1-\alpha$ addressing Problem \ref{prob:main}.

\begin{ex}[Extracting Responses and Frequency Scores]\label{ex2}
Consider the NL task described in Example~\ref{ex1}. Suppose that at the beginning of iteration $k=3$ of Alg.~\ref{alg:translation}, the partially constructed LTL formula is $\phi(2)=\lozenge($, which is used to construct the prompt $\ell(3)$. We illustrate how the function $\texttt{GetResp}(\psi,\ell(3),m,\zeta)$ computes a set of semantically distinct responses along with their frequencies.
Let $m=5$ and $\zeta=0.75$. Querying the LLM $\psi$ using $\texttt{Prompt}(\psi,\ell(3))$ for $m$ times [line~\ref{algo2:prompt}, Alg.~\ref{algo2:algorithm}] produces the following set of responses $\hat{\ccalS}^\psi(3)=$\{p\_red\_box, p\_red\_box, p\_red\_package, p\_green\_bottle, pre\%blocks\} after using the $\texttt{Prompt}(\psi,\ell(3))$. The last entry 'pre\%blocks' is then \textit{pruned} away using the \texttt{Prune} function, since it does not follow the rules requiring the action of pick up to be of the form p\_lmk 
[line \ref{algo2:prune}, Alg. \ref{algo2:algorithm}]. This gives us the pruned set $\ccalS^\psi(3)=$\{p\_red\_box, p\_red\_box, p\_red\_package, p\_green\_bottle\}.
The frequencies of the unique responses are then computed  as $F^{\psi}(\text{p\_red\_box}|
\ccalS(3))=0.5$, $F^{\psi}(\text{p\_red\_package}|
\ccalS(3))=0.25$, and $F^{\psi}(\text{p\_green\_bottle}|
\ccalS(3))=0.25$ [line~\ref{algo2:freqCalc}, Alg.~\ref{algo2:algorithm}]. Next, pairwise semantic similarities are evaluated using $\texttt{S.S}$ [line~\ref{algo2:SS}, Alg.~\ref{algo2:algorithm}]. In particular, consider the case where $\texttt{S.S}(\text{p\_red\_box},\text{p\_red\_package})=0.8>\zeta$, while similarities involving $\text{p\_green\_bottle}$ are below the threshold $\zeta$. Since $\text{p\_red\_box}$ appears more frequently than $\text{p\_red\_package}$, $\text{p\_red\_package}$ is replaced with $\text{p\_red\_box}$, yielding the following set of responses $\ccalS^\psi(3)=\{\text{p\_red\_box}, \text{p\_red\_box}, \text{p\_red\_box}, \text{p\_green\_bottle}\}$.
Applying the \texttt{Unique} function produces the set $\bar{\ccalS}^\psi(3)=\{\text{p\_red\_box},\text{p\_green\_bottle}\}$ [line~\ref{algo2:unique}, Alg.~\ref{algo2:algorithm}]. The updated frequencies are then $F^{\psi}(\text{p\_red\_box}|
\ccalS(3))=0.75$ and $F^{\psi}(\text{p\_green\_bottle}|
\ccalS(3))=0.25$ [line~\ref{algo:reComputeFreq}, Alg.~\ref{algo2:algorithm}]. The set $\bar{\ccalS}^\psi(3)$ together with these frequencies constitutes the output of $\texttt{GetResp}(\psi,\ell(3),m,\zeta)$.
\end{ex}

\begin{figure}[t] 
\centering
\includegraphics[width=\linewidth]{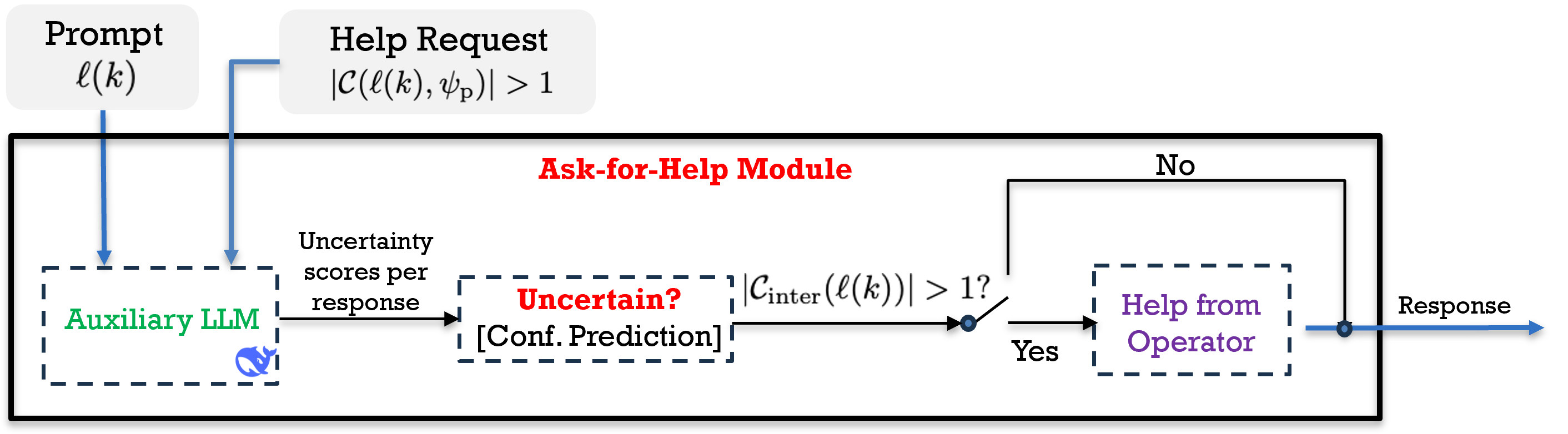}\vspace{-0.5cm}
\caption{Overview of the help module triggered due to 
a non-singleton prediction set $\ccalC(\ell(k),\psi_{\text{p}})$. The auxiliary LLM is then queried to produce its own prediction set $\ccalC(\ell(k),\psi_{\text{aux}})$. If the intersection of the two sets, denoted by $\ccalC_{\text{inter}}(\ell(k))$ (see \eqref{eq:intersectionSets}), is non-singleton—signaling high uncertainty—a human operator is asked to select the correct option from the elements of $\ccalC_{\text{inter}}(\ell(k))$.}\vspace{-0.2cm}
\label{fig:helpModule}
\end{figure}

\subsection{Asking for Help: Querying the Auxiliary LLM and a User \label{sec:auxLLM}}
In what follows, we present the help module that is invoked to handle high-uncertainty cases determined by the size of $\ccalC(\ell(k),{\psi_\text{p}})$; see Figure \ref{fig:helpModule}.
As discussed earlier, when $|\ccalC(\ell(k),\psi_{\text{p}})|>1$, our algorithm invokes the auxiliary LLM $\psi_{\text{aux}}$ to address the $k$-th QA problem using a prompt of the same structure and the same content in part (d) as $\ell(k)$, and the same process applied to the primary model. For simplicity of notation we use $\ell(k)$ to denote the prompt used to query either model, {since they have the same structure and are constructed from the same input $\ccalA$, $\xi$, and $\phi(k-1)$.}
This produces a prediction set $\ccalC(\ell(k),\psi_{\text{aux}})$ that contains the ground truth with user-specified confidence [line \ref{algo:auxPred}, Alg. \ref{alg:translation}]; {we defer its construction to Section \ref{sec:cp}.}

Next, we compute the intersection of the prediction sets generated by the primary and the auxiliary model, i.e.,
\begin{equation}\label{eq:intersectionSets}
    {\ccalC}_{\text{inter}}(\ell(k))=\ccalC(\ell(k),{\psi_\text{p}})\cap\ccalC(\ell(k),\psi_{\text{aux}}).
\end{equation}
Thus, $\ccalC_{\text{inter}}(\ell(k))$ contains only the responses that are shared by both models [line \ref{algo:intersection}, Alg. \ref{alg:translation}]. If $|{\ccalC}_{\text{inter}}(\ell(k))|=1$, then $s(k)\in{\ccalC}_{\text{inter}}(\ell(k))$ is selected as the response [line \ref{algo:finalPick}, Alg. \ref{alg:translation}]. If $|{\ccalC}_{\text{inter}}(\ell(k))|>1$, the algorithm asks for help from a human operator to select a response $s(k)$ from the prediction set ${\ccalC}_{\text{inter}}(\ell(k))$ given the prompt $\ell(k)$ [line \ref{algo1:HelpHuman2}, Alg. \ref{alg:translation}]; see also Remark \ref{rem:multiFeasFinal}. If the prediction set ${\ccalC}_{\text{inter}}(\ell(k))$ does not contain any correct responses (e.g., if both $\psi_{\text{p}}$ and $\psi_{\text{aux}}$ generate the same incorrect response), or if ${\ccalC}_{\text{inter}}(\ell(k))=\emptyset$ (none of the $m_k$ generated responses are valid, i.e., $\ccalS^{\psi_{\text{aux}}}(k)=\emptyset$, {or $\ccalC(\ell(k),\psi_{\text{p}})$ and $\ccalC(\ell(k),\psi_{\text{aux}})$ are mutually exclusive}), the user can halt the translation process, marking it as a `failure' of our translation algorithm [line \ref{algo1:halt}, Alg. \ref{alg:translation}]; see Theorem \ref{thm1}.

\vspace{-0.3cm}
\subsection{Constructing Prediction Sets using Conformal Prediction}\label{sec:cp}

In this section, we discuss how the prediction sets $\mathcal{C}(\ell(k),{\psi_\text{p}})$ and ${\ccalC}(\ell(k),\psi_{\text{aux}})$, introduced in Section \ref{sec:primaryLLM} and Section \ref{sec:auxLLM} respectively, are constructed using CP, based on a user-defined translation success rate $1 - \alpha$ (see Problem \ref{prob:main}). The application of CP to {a single} LLM for single-step QA tasks in open vocabulary settings has been explored in \cite{su2024api,kaur2024addressing}. 
We extend this approach to address \textit{multi-step interdependent QA tasks in open-vocabulary settings}, and the usage of \textit{more than one LLM}, both of which are essential for ConformalNL2LTL. {Background on CP is provided in Appendix \ref{appendix:CP}.}

We would like to note that we cannot apply the approach in \cite{su2024api} to compute prediction sets
$\ccalC(\ell(k),{\psi_\text{p}})$, and ${\ccalC}(\ell(k),\psi_{\text{aux}})$ at each iteration $k$, as this violates the i.i.d. assumption required to apply CP. The reason is that the prompt $\ell(k)$ depends on the previous prompt $\ell(k-1)$, for all $k\in\{2,\ldots,H\}$, by construction. Inspired by \cite{ren2023robots,lindemann2023safe,wang2024safe}, we address this challenge by lifting the data to sequences. Specifically, $\ccalD$ will be transformed into an equivalent distribution $\ccalD'$ over data sequences and, more precisely, sequences of prompts $\ell(k)$. Then, we will apply CP by constructing (i) a calibration dataset and (ii) designing a non-conformity score (NCS) modeling the joint translation error of the primary and the auxiliary models. Throughout this section, for simplicity of presentation, we assume that there exists a unique ground truth formula $\phi$ corresponding to each scenario drawn from $\ccalD$; relaxing this assumption is discussed in Remark \ref{rem:multiFeasFinal} {and more formally in Appendix \ref{app:proof}}.
Next, we discuss our approach. We start with defining the distribution $\ccalD'$ used to lift translation scenarios $\sigma\sim\ccalD$ to data sequences. 

\textbf{Distribution $\ccalD'$:} Consider a translation scenario $\sigma\sim\ccalD$ that is transformed into a sequence of prompts $\ell(k)$, i.e.,
\begin{align}\label{equ:seqprompt}
    \bar\ell=\ell(1),\ldots,\ell(k),\ldots,\ell(H).
\end{align}
Each prompt $\ell(k)$ in the sequence $\bar\ell$  has the structure defined in \ref{sec:primaryLLM}, {where part (d) is constructed using the ground truth responses of all steps up to $k-1$.} 
The true response associated with $\ell(k)$ is selected from the following set containing the shared responses of both models:
\begin{equation}\label{eq:barS}
    \Bar{\ccalS}(k)=\Bar{\ccalS}^{\psi_\text{p}}(k)\cap\Bar{\ccalS}^{\psi_{\text{aux}}}(k),
\end{equation}
where $[\Bar{\ccalS}^{\psi_p}(k),F^{\psi_{p}}(k)]=\texttt{GetResp}({\psi_\text{p}},\ell(k),m,\zeta)$ and $[\Bar{\ccalS}^{\psi_{\text{aux}}}(k),F^{\psi_{\text{aux}}}(k)]=\texttt{GetResp}(\psi_{\text{aux}},\ell(k),m,\zeta)$. {As mentioned in Section \ref{sec:auxLLM}, the prompts $\ell(k)$ used for the primary and auxiliary LLMs have the same structure and the same content in part (d), differing only in the specific content of parts (a)-(c).} 
If the true response is not contained in $\Bar{\ccalS}(k)$ or if $\Bar{\ccalS}(k)=\emptyset$, then we type in the correct response. Thus, $\ccalD$ is converted into an equivalent distribution $\ccalD'$ over
sequences of prompts, where each sequence of prompts is constructed by selecting the true responses at each decision step $k$ as
discussed above. Note that the set $\Bar{\ccalS}(k)$ depends on the parameters $m$ and $\zeta$; consequently, $\ccalD'$ is conditioned on the values of these parameters.

\textbf{Constructing a Calibration Dataset:} As discussed earlier, applying CP requires a calibration dataset, which is constructed as follows. First, we sample $D>0$ i.i.d. scenarios $\{\sigma_{i,\text{calib}}\}_{i=1}^D$ from $\mathcal{D}$ and construct the following sequence of $H_i$ prompts for each $i \in \{1, \ldots, D\}$:
\begin{align}\label{equ:prompt}
    \bar\ell_{i,\text{calib}}=\ell_{i,\text{calib}}(1),\ldots,\ell_{i,\text{calib}}(H_i),
\end{align}
where each prompt $\ell_{i,\text{calib}}(k)$ in $\bar\ell_{i,\text{calib}}$  has the structure defined in \ref{sec:primaryLLM}, {and contains the ground truth responses of all previous steps in part (d)}. The corresponding true LTL formula is: 
\begin{align} \phi_{i,\text{calib}}=s_{i,\text{calib}}(1),\ldots,s_{i,\text{calib}}(H_i),\label{eq:calibLTL}
\end{align} 
where $s_{i,\text{calib}}(k)$ refers to the correct response to $\ell_{i,\text{calib}}(k)$. This correct formula is constructed by computing $\Bar{\ccalS}(k)$ defined in \eqref{eq:barS}.
%
Then, we select the correct response from $\Bar{\ccalS}(k)$. If the correct response in not in $\Bar{\ccalS}(k)$, or if $\Bar{\ccalS}(k)=\emptyset$, we type in the correct response; the frequency of that response is $0$. This process gives rise to a calibration set $\mathcal{M}_{\text{cal}}=\left\{\left(\bar\ell_{i,\text{calib}},\phi_{i,\text{calib}}\right)\right\}_{i=1}^D$ {where $\bar\ell_{i,\text{calib}}\sim\ccalD'$ by construction.}

\textbf{Non-Conformity Score (NCS):} Next, we define a non-conformity score (NCS) modeling the worst-case 
translation error of the primary and the auxiliary models across all steps during the translation process. The NCS will be computed for all calibration sequences. 
Specifically, we define the NCS of the $i$-th calibration sequence as 
\begin{equation}\label{eq:NCS}
    \bar r_i = 1-\bar F(\phi_{i,\text{calib}}|\bar\ell_{i,\text{calib}})
\end{equation}
where 
\begin{equation}\label{eq:barF}
    \bar F(\phi_{i,\text{calib}}|\bar\ell_{i, \text{calib}})=\min_{k\in\{1,\ldots,H_i\},\psi\in\{\psi_p,\psi_{\text{aux}}\}} F^\psi(s_{i,\text{calib}}(k)|\ccalS(k)),
\end{equation}
where $\phi_{i,\text{calib}}$ is as defined in \eqref{eq:calibLTL}, and $F^\psi(s_{i,\text{calib}}(k)|\ccalS(k))$ is the frequency of the correct response $s_{i,\text{calib}}(k)$ computed 
{based on $[\Bar{S}^\psi, F^\psi]=\texttt{GetResp}(\psi,\ell_{i,\text{calib}}(k),m,\zeta)$} for model $\psi\in\{\psi_{\text{p}}, \psi_{\text{aux}}\}$. In words, \eqref{eq:barF} computes the lowest frequency assigned to the correct response across all translation steps $k$ and both models for the $i$-th calibration sequence.
%

\textbf{Constructing Prediction Sets:} Consider an unseen test/validation scenario $\sigma_{\text{test}}\sim\ccalD$ \textcolor{black}{which is used to generate a} sequence of prompts:
\begin{align}\label{equ:prompt}
\bar\ell_{\text{test}}=\ell_{\text{test}}(1),\ldots,\ell_{\text{test}}(H_{\text{test}}).
\end{align}
%
If for all $k>1$, $\ell_{\text{test}}(k)$ is constructed using the true response for $\ell_{\text{test}}(k-1)$—either chosen from $\Bar{\ccalS}(k-1)$ (defined in \eqref{eq:barS}) or supplied by the user when $\Bar{\ccalS}(k-1)=\emptyset$ or does not contain the correct response—then the calibration and validation sequences are i.i.d. according to the distribution $\ccalD'$.  
Then, CP generates a prediction set $\bar{\ccalC}(\bar{\ell}_{\text{test}})$ of LTL formulas, containing the correct one, denoted by $\phi_{\text{test}}$,  with probability greater than $1-\alpha$, i.e.,
\begin{equation}\label{eq:CP1}
P_{\bar{\ell}_{\text{test}}\sim\ccalD'}(\phi_{\text{test}}\in \bar{\mathcal{C}}(\bar{\ell}_{\text{test}})) \geq 1-\alpha.
\end{equation}
\textcolor{black}{This prediction set is defined as follows:
\begin{equation}\label{eq:pred3}
\bar{\mathcal{C}}(\bar{\ell}_{\text{test}}) = \{\phi~|~ \bar{F}(\phi~|~ \bar{\ell}_{\text{test}}) \geq 1-\bar{q} \},
\end{equation} where $\bar{q}$ is the $\frac{\lceil(D+1)(1-\alpha)\rceil}{D}$ empirical quantile of $\{\bar{r}_i\}_{i=1}^D$.
}\footnote{{If the quantile value $\bar q=1$, then the prediction set will contain all possible combinations of LLM token outputs, which is a very large finite set. In practice, this may happen if there exist multiple calibration sequences with NCS $\bar{r}_i=1$ (i.e., at least one of the responses during the construction of the ground truth formula was provided by a user) and a sufficiently low $\alpha$.}}.

%
%

\begin{algorithm}[t]\label{algo3:algorithm}
\footnotesize
\caption{$\texttt{PredSet}$}\label{algo3:algorithm}
\begin{algorithmic}[1]
\State \textbf{Input:} $\ell(k)$, LLM $\psi$, quantile $\Bar{q}$, semantic threshold $\zeta$
\State Initialize $\ccalC(\ell(k),\psi)=\emptyset$
\State [$\Bar{\ccalS}^\psi(k),F^\psi(k)$]=$\texttt{GetResp}(\psi,\ell(k),m,\zeta)$ (see Alg. \ref{algo2:algorithm})\label{algo3:getresp}
\For{$s(k)\in\Bar{\ccalS}^\psi(k)$}\label{algo3:for}
\If{$F^\psi(s(k))\geq1-\Bar{q}$}\label{algo3:checkValid}
\State $\ccalC(\ell(k),\psi)=\ccalC(\ell(k),\psi)\cup\{s(k)\}$\label{algo3:addtoSet}
\EndIf
\EndFor
\State\textbf{Output: }$\ccalC(\ell(k),\psi)$
\end{algorithmic}
\end{algorithm} 

\textbf{On-the-fly Construction of Prediction Sets: } Notice that $\bar{\mathcal{C}}(\bar{\ell}_{\text{test}})$ in \eqref{eq:pred3} is constructed after the entire sequence $\bar{\ell}_{\text{test}}=\ell_{\text{test}}(1),\dots,\ell_{\text{test}}(H_{\text{test}})$ is obtained.
However, at test time, we do not see the entire sequence of prompts all at once; instead, the contexts ${\ell}_{\text{test}}(k)$ are revealed sequentially over the steps $k$. 

Thus, we construct the prediction set $\bar{\mathcal{C}}(\bar{\ell}_{\text{test}})$ on-the-fly and incrementally, using only the current and past information; see Alg. \ref{algo3:algorithm}. 
At every step $k\in\{1,\dots,H_{\text{test}}\}$, we construct the following local prediction set.


\begin{align}\label{eq:predSetStep}
\ccalC(\ell_{\text{test}}(k))=\begin{cases}
    \ccalC(\ell_{\text{test}}(k),{\psi_\text{p}}), &\text{if }|\ccalC(\ell_{\textcolor{black}{\text{test}}}(k),{\psi_\text{p}})|=1 \\
    \ccalC_{\text{inter}}(\ell_{\textcolor{black}{\text{test}}}(k)),&\text{otherwise.}
\end{cases}
\end{align}
\textcolor{black}{In \eqref{eq:predSetStep}, $\ccalC_{\text{inter}}(\ell_{\textcolor{black}{\text{test}}}(k))$ is defined as in  \eqref{eq:intersectionSets}, i.e., $\ccalC_{\text{inter}}(\ell_{\textcolor{black}{\text{test}}}(k))=\ccalC(\ell_{\text{test}}(k),{\psi_\text{p}})\cap\ccalC(\ell_{\text{test}}(k),\psi_{\text{aux}})$, where $\mathcal{C}({\ell}_{\text{test}}(k),{\psi_\text{p}})$ and $ \ccalC({\ell}_{\text{test}}(k),\psi_{\text{aux}})$ are prediction sets produced by the primary and the auxiliary models, respectively.}

Specifically, the set $\mathcal{C}({\ell}_{\text{test}}(k),{\psi_\text{p}})$ is defined as follows
\begin{align}\label{eq:pred3local}
    \nonumber\mathcal{C}&({\ell}_{\text{test}}(k),{\psi_\text{p}})\\&=\{s_{\text{test}}(k)\in\Bar{\mathcal{S}}^{\psi_p}_{\text{test}}(k)~|~F^{\psi_p}(s_{\text{test}}(k)|\ccalS(k))\ge1-\bar{q}\},
\end{align}
{where $\Bar{\mathcal{S}}^{\psi_p}_{\text{test}}(k)$ is computed by $[\Bar{\mathcal{S}}^{\psi_p}_{\text{test}}(k),F^{\psi_p}(k)]=\texttt{GetResp}({\psi_\text{p}},\ell_{\text{test}}(k),m,\zeta)$ [lines \ref{algo3:getresp}-\ref{algo3:addtoSet}, Alg. \ref{algo3:algorithm}].}

Similarly, the prediction set $\ccalC({\ell}_{\text{test}}(k),\psi_{\text{aux}})$ associated with the auxiliary model $\psi_{\text{aux}}$, is defined as 
\begin{align}\label{eq:auxPredSet}
    \nonumber\ccalC&({\ell}_{\text{test}}(k),\psi_{\text{aux}})\\&=\{s_{\text{test}}(k)\in\Bar{\mathcal{S}}^{\psi_{\text{aux}}}_{\text{test}}(k)~|~F^{\psi_{\text{aux}}}(s_{\text{test}}(k)|\ccalS(k))\geq 1-\bar{q}\},
\end{align}
{where $\Bar{\mathcal{S}}^{\psi_{\text{aux}}}_{\text{test}}(k)$ is computed by $[\Bar{\mathcal{S}}^{\psi_{\text{aux}}}_{\text{test}}(k),F^{\psi_{\text{aux}}}_{\text{test}}(k)]=\texttt{GetResp}(\psi_{\text{aux}},\ell_{\text{test}}(k),m,\zeta)$ [line \ref{algo3:getresp}-\ref{algo3:addtoSet}, Alg. \ref{algo3:algorithm}].}
%
%
The prediction sets $ {\mathcal{C}}({\ell}_{\text{test}}(k),{\psi_\text{p}})$ and $ {\mathcal{C}}({\ell}_{\text{test}}(k),{\psi_\text{aux}})$ are used in Section \ref{sec:primaryLLM} to select $s(k)$ [line \ref{algo1:pred_set}-\ref{algo1:pickAction}, Alg. \ref{alg:translation}].
As it will be shown in Sec. \ref{sec:theory}, it holds that 
\begin{align}\label{eq:CP_local}
    P_{\bar{\ell}_{\text{test}}\sim\ccalD'}(\phi_{\text{test}}\in\hat{\ccalC}(\Bar{\ell}_{\text{test}}))\geq 1-\alpha,
\end{align}
where
\begin{equation}\label{eq:causalPred3}
    \hat{\ccalC}(\bar{\ell}_{\text{test}})={\mathcal{C}}({\ell}_{\text{test}}(1))\times\dots\times{\mathcal{C}}({\ell}_{\text{test}}(k))\dots\times{\mathcal{C}} ({\ell}_{\text{test}}(H_{\text{test}})).
\end{equation} 


\begin{rem}[Construction of Prediction Sets]
    We emphasize that construction of the sets $\ccalC(\ell_{{\text{test}}}(k))$ in \eqref{eq:predSetStep} does not require knowledge of the true responses during validation time. However, if all prompts $\ell_{{\text{test}}}(k)$ have been constructed using the (unknown) true response at every step $k$, then \eqref{eq:CP_local} holds (see Section \ref{sec:theory}). This result will be used in Section \ref{sec:theory} to show that ConformalNL2LTL achieves user-specified $1-\alpha$ translation success rates {(see Theorem \ref{thm1} and its proof)}.
\end{rem}


\vspace{-0.2cm}
\section{Translation Success Rate Guarantees}\label{sec:theory}
\vspace{-0.1cm}

In this section, we show that Alg. \ref{alg:translation} achieves user-specified $1-\alpha$ translation success rates as required in Problem \ref{prob:main}. To show this, we need to first state the following result.

\begin{prop}\label{prop:new}
    Consider a sequence of prompts $\Bar{\ell}_{\text{test}}:=\ell_{\text{test}}(1),\ldots,\ell_{\text{test}}(H_{\text{test}})\sim\ccalD'$, i.e., $\ell_{\text{test}}(k)$ contains the true response $s_{\text{test}}(k-1)$, for all $k\in\{1,\ldots,H_{\text{test}}\}$. It holds that $P_{\bar{\ell}_{\text{test}}\sim\ccalD'}(\phi_{\text{test}}\in\hat{\ccalC}(\Bar{\ell}_{\text{test}}))\geq1-\alpha$, where $\hat{\ccalC}(\Bar{\ell}_{\text{test}})$ is as defined in \eqref{eq:causalPred3}.
\end{prop}

\begin{proof}
    To show that $P(\phi_{\text{test}}\in\hat{\ccalC}(\Bar{\ell}_{\text{test}}))\geq1-\alpha$, 
    it suffices to show that,
    \begin{align}\label{eq:subseteqprop}
        \Bar{\ccalC}(\Bar{\ell}_{\text{test}})\subseteq\hat{\ccalC}(\Bar{\ell}_{\text{test}})
    \end{align}due to \eqref{eq:CP1}. To show \eqref{eq:subseteqprop}, we first introduce the following set of LTL formulas:
    \begin{align}\label{eq:interformulaprop}
        \ccalC_{\text{inter}}(\Bar{\ell}_{\text{test}}):=\ccalC_{\text{inter}}(\ell_{\text{test}}(1))\times\ldots\times\ccalC_{\text{inter}}(\ell_{\text{test}}(H_{\text{test}})),
    \end{align}where  $\ccalC_{\text{inter}}(\ell_{\text{test}}(k))=\ccalC(\ell_{\text{test}}(k),\psi_{\text{p}})\cap\ccalC(\ell_{\text{test}}(k),\psi_{\text{aux}})$ as defined in \eqref{eq:intersectionSets}. In what follows, we will show that (A) $\Bar{\ccalC}(\Bar{\ell}_{\text{test}})={\ccalC}_{\text{inter}}(\Bar{\ell}_{\text{test}})$ and (B) $\hat{\ccalC}(\Bar{\ell}_{\text{test}})\supseteq\ccalC_{\text{inter}}(\Bar{\ell}_{\text{test}})$. Combining (A) and (B) yields \eqref{eq:subseteqprop}, completing the proof.
    
    \textit{Proof of Part (A): }To prove (A), it suffices to show that for any LTL formula $\phi$, if $\phi\in\Bar{\ccalC}(\Bar{\ell}_{\text{test}})\iff\phi\in\ccalC_{\text{inter}}(\Bar{\ell}_{\text{test}})$, where $\phi=s(1)\ldots s(k)\ldots s(H_{\text{test}}).$

For any $\phi\in\Bar{\ccalC}(\Bar{\ell}_{\text{test}})$, due to \eqref{eq:pred3},
    \begin{align}
        \nonumber&\Bar{F}(\phi|\Bar{\ell}_{\text{test}})\geq1-\Bar{q}\\
        \nonumber\iff&\min_{k\in\{1,\ldots,H_{\text{test}}\},\psi\in\{\psi_{\text{p}},\psi_{\text{aux}}\}}F^{\psi}(s(k))\geq1-\Bar{q},\\
        \nonumber\iff&F^{\psi}(s(k))\geq1-\Bar{q},\forall k\in\{1,\ldots,H_{\text{test}}\}, \forall\psi\in\{\psi_{\text{p}},\psi_{\text{aux}}\},\\
        \nonumber\iff&\{F^{\psi_{\text{p}}}(s(k))\geq1-\Bar{q}\} \texttt{ and }\{F^{\psi_{\text{aux}}}(s(k))\geq1-\Bar{q}\},\\
        \nonumber&\forall k\in\{1,\ldots,H_{\text{test}}\},\\
        \nonumber\iff&\{s(k)\in\ccalC(\ell_{\text{test}}(k),\psi_{\text{p}})\}\nonumber\texttt{ and }\{s(k)\in\ccalC(\ell_{\text{test}}(k),\psi_{\text{aux}})\},\\\nonumber&\forall k\in\{1,\ldots,H_{\text{test}}\},\\
        \nonumber\iff&s(k)\in\ccalC(\ell_{\text{test}}(k),\psi_{\text{p}})\cap\ccalC(\ell_{\text{test}}(k),\psi_{\text{aux}}),\\&\forall k\in\{1,\ldots,H_{\text{test}}\}.\label{eq:propstep}
    \end{align}Due to \eqref{eq:propstep}, we get that $\phi\in\ccalC_{\text{inter}}(\Bar{\ell}_{\text{test}}).$ These implications also hold in the reverse direction: for any $\phi$, if $\phi\in\ccalC_{\text{inter}}(\Bar{\ell}_{\text{test}})$, then $\phi\in\Bar{\ccalC}(\Bar{\ell}_{\text{test}})$, completing the proof of Part (A).

\textit{Proof of Part (B): }This result holds by construction of $\ccalC_{\text{inter}}(\Bar{\ell}_{\text{test}})$ in \eqref{eq:interformulaprop} and $\hat{\ccalC}(\Bar{\ell}_{\text{test}})$ in \eqref{eq:causalPred3}. Specifically, observe that from \eqref{eq:intersectionSets} and \eqref{eq:predSetStep},
    \begin{align}\label{eq:supsetstepprop}
        \ccalC(\ell_{\text{test}}(k))\supseteq\ccalC_{\text{inter}}(\ell_{\text{test}}(k)),\forall k\in\{1,\ldots,H_{\text{test}}\}.
    \end{align}Thus, since the set $\hat{\ccalC}(\Bar{\ell}_{\text{test}})$ is defined in \eqref{eq:causalPred3} as the cartesian product of all set $\ccalC(\ell_{\text{test}}(k))$ and, similarly, $\ccalC_{\text{inter}}(\Bar{\ell}_{\text{test}})$ is defined as the cartesian product of all sets $\ccalC_{\text{inter}}(\ell_{\text{test}}(k))$ in \eqref{eq:interformulaprop}, we get $\hat{\ccalC}(\Bar{\ell}_{\text{test}})\supseteq\ccalC_{\text{inter}}(\Bar{\ell}_{\text{test}})$ as a direct consequence of \eqref{eq:supsetstepprop}, completing the proof of Part (B).
\end{proof}

\vspace{-0.4cm}
\textcolor{black}{\begin{theorem}[Translation Success Rates]\label{thm1}
    \textcolor{black}{Consider translation scenarios $\sigma_{\text{test}}\sim\ccalD$ and any choice of the number of queries $m>1$ and semantic similarity threshold $\zeta>0$. Assume that $m$ and $\zeta$ remain the same during calibration and validation time.} (a) If, for each scenario $\sigma_{\text{test}}\sim\ccalD$ (where $\ccalD$ is unknown), the prediction sets $\ccalC(\ell_{\text{test}}(k))$, $\forall k\in\{1,\ldots,H_{\text{test}}\}$ are constructed \textit{on-the-fly} with a coverage level of $1-\alpha$, 
    and the LLM asks for help from a \textcolor{black}{benign user} (who chooses the ground truth response when it is found in $\ccalC(\ell_{\text{test}}(k))$) when $|\ccalC(\ell_{\text{test}}(k))|>1$, \textcolor{black}{then ${P}_{\sigma\sim\ccalD}(\phi\equiv\xi_{\text{test}})\geq1-\alpha$, where $\phi$ is constructed as described in Section \ref{sec:primaryLLM}}. (b) If $\Bar{F}(\phi_{\text{test}}|\Bar{\ell}_{\text{test}})$, used in \eqref{eq:pred3}, models true conditional probabilities, then the help rate is minimized among possible prediction schemes that achieve $1-\alpha$ translation success rates. 
    \footnote{{The \textit{help rate} is defined as the percentage of steps $k$, where $|\ccalC(\ell_{\text{test}}(k))|>1$, during the construction of formulas corresponding to test scenarios $\sigma_{\text{test}}$.}}.
\end{theorem}}

\begin{proof}
(a) To show this result, we first need to define three Boolean variables $C_1$, $C_2$, and $C_3$
modeling three mutually exclusive and exhaustive events. 

\textit{Variable $C_1$:} The Boolean variable $C_1$ is true if at any time step $k$, all prediction sets $\mathcal{C}({\ell}_{\text{test}}(k))$ are singleton and contain the true response $s_{\text{test}}(k)$, for all $k\in\{1,\dots,H_{\text{test}}\}$, i.e., 
\begin{align}
    C_1=\bigwedge_{k=1}^{H_{\text{test}}}\Big[\mathcal{C}\big({\ell}_{\text{test}}(k)\big)=\{s_{\text{test}}(k)\}\Big]\label{eq:case1}
\end{align}
%
\textit{Variable $C_2$:} The Boolean variable $C_2$ is true when all prediction sets contain the true response as in $C_1$ but at least one prediction set is non-singleton, i.e.,
\begin{align}
    C_2=\nonumber&\bigg\{\bigwedge_{k=1}^{H_{\text{test}}}\Big[s_{\text{test}}(k)\in\mathcal{C}\big({\ell}_{\text{test}}(k)\big)\Big]\bigg\}\\&\wedge \bigg\{\bigvee_{k=1}^{H_{\text{test}}}\Big[|\mathcal{C}\big({\ell}_{\text{test}}(k)\big)|>1\Big]\bigg\}\label{eq:case2}
\end{align}
\textit{Variable $C_3$:} The Boolean variable $C_3$ is true if at least one of the prediction sets $\mathcal{C}({\ell}_{\text{test}}(k))$ does not contain the true response, regardless of its cardinality, i.e.

\begin{align}
    C_3=\bigvee_{k=1}^{H_{\text{test}}}\Big[s_{\text{test}}(k)\notin\mathcal{C}\big({\ell}_{\text{test}}(k)\big)\Big]\label{eq:case3}
\end{align}
Observe that $C_3$ is also true if for some $k\in\{1,\ldots,H_{\text{test}}\}$, we have $\ccalC(\ell_{\text{test}}(k))=\emptyset$, since in that case it trivially holds $s_{\text{test}}\not\in\ccalC(\ell_{\text{test}}(k))$.

Observe that to prove (a), it suffices to show that $P(C_{12})\geq 1-\alpha$, where $C_{12}=C_1 \vee C_2$. This is because if either of these variables is true, then all prediction sets $\mathcal{C}({\ell}_{\text{test}}(k))$ contain the true response $s_{\text{test}}(k)$. Specifically, if all prediction sets are singleton (event $C_1$), then, by construction, ConformalNL2LTL selects $s_{\text{test}}(k)$ at every step $k$.
If at least one prediction set is non-singleton (event $C_2$), the user will still select the correct response, under the assumption of benign users. In contrast, if $C_3$ is true, translation failure is unavoidable: if prediction sets are non-singleton, the user halts the process; if they are singleton, the wrong response is selected. Thus it suffices to bound the probability of $C_{12}$.

In what follows, we show that $P(C_{12})\geq 1-\alpha$. Specifically, we will show that $P(C_{12})=P_{\Bar{\ell}_{\text{test}}\sim\ccalD'}(\phi_{\text{test}}\in\hat{\ccalC}(\Bar{\ell}_{\text{test}}))$ which is greater than $1-\alpha$ due to Proposition \ref{prop:new}.  Notice that $P_{\Bar{\ell}_{\text{test}}\sim\ccalD'}(\phi_{\text{test}}\in\hat{\ccalC}(\Bar{\ell}_{\text{test}}))$ is the probability that the ground truth response $s_{\text{test}}(k)$ belongs to $\mathcal{C}({\ell}_{\text{test}}(k))$ for all $k$ (see \eqref{eq:causalPred3}), assuming that all prompts $\ell_{\text{test}}(k)$ include the ground truth responses of all previous steps till $k-1$. We can therefore expand $P_{\Bar{\ell}_{\text{test}}\sim\ccalD'}(\phi_{\text{test}}\in\hat{\ccalC}(\Bar{\ell}_{\text{test}}))$ and rewrite it as follows:
\begin{align}\label{eq:propInput}
        \nonumber P_{\Bar{\ell}_{\text{test}}\sim\ccalD'}&\big(\phi_{\text{test}}\in\hat{\ccalC}(\Bar{\ell}_{\text{test}})\big)\\=&P\bigg(\bigwedge_{k=1}^{H_{\text{test}}}\Big(s_{\text{test}}(k)\in\ccalC\big(\ell_{\text{test}}(k)\big)|\bigwedge_{r=1}^{k-1}B_r\Big)\bigg)\geq1-\alpha.
\end{align}
where $B_r$ is Boolean variable that is true if $\ell_{\text{test}}(r)$ is constructed with the ground truth responses {and the inequality holds due to Proposition \ref{prop:new}.}

Now observe that $C_{12}$ can be written as 
\begin{align}
    C_{12}=\bigg\{\bigwedge_{k=1}^{H_{\text{test}}}\Big[s_{\text{test}}(k)\in\mathcal{C}\big({\ell}_{\text{test}}(k)\big)\Big]\bigg\}\label{eq:case12}
\end{align}
As discussed above, if $C_{12}$ is true, it means that every time step $k$, ConformalNL2LTL selects the correct response. This equivalently means that every prompt $\ell_{\text{test}}(k)$ has been constructed using the correct responses. {And since, \eqref{eq:case12} is a conjunction over all $k\in\{1,\ldots,H_{\text{test}}\}$, it can be written equivalently as}
    \begin{align}
    C_{12}=\Bigg\{\bigwedge_{k=1}^{H_{\text{test}}}\Big[s_{\text{test}}(k)\in\mathcal{C}\big({\ell}_{\text{test}}(k)\big)|\bigwedge_{r=1}^{k-1}B_r\Big]\Bigg\}.\label{eq:case12_b}
\end{align}
%
{This equivalence holds because, at every step $k$, $B_{r}$ will always be true $\forall r\in\{1,\ldots,k-1\}$ {as $s_{\text{test}}(r)\in\mathcal{C}({\ell}_{\text{test}}(r))$, $\forall r\in\{1,\ldots,k-1\}$.}}
Comparing \eqref{eq:case12_b} and \eqref{eq:propInput}, we get that $P(C_{12})\geq 1-\alpha$ as required. We emphasize that the guarantee in Proposition~\ref{prop:new}, and, consequently, the guarantee $P(C_{12})\geq 1-\alpha$,  holds for any fixed value of $m$ and $\zeta$. These parameters must remain unchanged during both calibration and validation; modifying them induces a distribution shift, since the induced distribution $\ccalD'$ is conditioned on them. {However, the exact value of $m$ and $\zeta$ does not affect correctness of these guarantees as this only changes the granularity of of the NCS \cite{angelopoulos2021gentle}. Specifically, coverage guarantees provided by CP depend on the accuracy of ranking the responses from lowest to highest model confidence, and not on the actual values of the NCS. Even if the scores fail to approximate this ranking, the marginal coverage still holds, albeit with potentially excessively large prediction sets \cite{angelopoulos2021gentle}.}


(b) This result holds directly due to Th. 1 in \cite{sadinle2019least}.\end{proof}

\begin{rem}[Multiple Semantically Equivalent Formulas]\label{rem:multiFeasFinal}
    The above results assume that $\ccalD$ generates task scenarios with a unique corresponding LTL formula. To relax this assumption, the following changes are made during calibration and test time. (i) During calibration, among all LTL formulas that are semantically equivalent to each other and to a given NL task, we select the one constructed by picking the {correct response} $s_{i,\text{calib}}(k)\in\Bar{\ccalS}(k)$ with the highest frequency $F^{\psi_{\text{p}}}(s_{i,\text{calib}}(k)|\ccalS(k))$ {(among all other correct responses in $\Bar{\ccalS}(k)$)} at each step $k$ as computed by $\psi_\text{p}$. If none of the responses generated is the ground truth, or valid, we type out the correct response with the frequency of the ground truth being $F^{\psi_{\text{p}}}(s_{i,\text{calib}}(k)|\ccalS(k))=0$ in such cases. 
    %
    %
    (ii) Given a validation NL task $\xi_{\text{test}}$, the LTL formula $\phi_{\text{test}}$ in {both Proposition \ref{prop:new} and} Theorem \ref{thm1} refers to the formula that is semantically equivalent to $\xi_{\text{test}}$ and has been constructed exactly as in (i). {The proofs of both Proposition \ref{prop:new} and Theorem \ref{thm1} remain the same.} (iii) When user help is required at test time, the user selects the correct response among the ones in the prediction set with the highest frequency according to $\psi_p$. {Extension of the CP analysis provided in Section \ref{sec:cp} to cases with multiple semantically equivalent formulas is presented in Appendix \ref{app:proof}.}
   
\end{rem}

%% file: experiment_v2.tex
In this section, we demonstrate the effectiveness of ConformalNL2LTL. Section \ref{sec:exp_nltask} describes the experimental setup and key implementation details. Section \ref{sec:alphaVal} empirically validates the translation success rate guarantees of Thm. \ref{thm1} and demonstrates the algorithm's help rate. The benefits of using an auxiliary LLM as well as the effect of the number $m$ of queries are shown in Section \ref{sec:auxModelHelp}. {Comparisons against existing NL-to-LTL translators are provided in Section \ref{sec:compareTranslate}.}
Section \ref{sec:comparePlanner} compares the performance of ConformalNL2LTL, incorporated with an existing planner \cite{luo2021abstraction}, with NL-based planners \cite{chen2023scalable}. Comparisons show that our method outperforms existing baselines as it can achieve desired translation performance while minimally asking for help from users.
Section \ref{sec:ood} evaluates robustness of our method in out-of-distribution (OOD) NL tasks. While probabilistic guarantees of Thm. \ref{thm1} may not hold in OOD settings,  our framework achieves satisfactory success rates with very little user help. Demonstrations on NL-encoded navigation and mobile manipulation tasks using ground robots are included in
\cite{video_demo}, where ConformalNL2LTL is integrated with an existing LTL planner \cite{luo2021abstraction}; see Fig. \ref{fig:env}. 
All case studies use GPT-4o as the primary model and Deepseek v3 as the auxiliary model. 
%

\subsection{Setting Up Translation Scenarios for Robot Tasks}\label{sec:exp_nltask}
\textbf{Robot \& Environment:} Our experiments focus on robot delivery and transportation tasks. The robot's action space is defined as $\mathcal{A}=$\{move to, pick up, put down, take a picture\} across all translation scenarios.
The environments vary in terms of objects (e.g., crates, boxes, cars, bottles, and traffic cones) as well as regions of interest (e.g., kitchens, warehouses, buildings, and parking lots) across our translation scenarios. Since multiple objects of the same semantic label (e.g., `box') may exist in the environment, we assign unique identifiers (e.g., `box 1', `box 2') to distinguish them. 

\textbf{Rules for APs:} As discussed in Section \ref{sec:primaryLLM}, part (a) of each prompt informs the LLM about how APs should be formulated indicating when a robot action is applied to one (if any) of the available landmarks. In our setup, we require the LLM to generate APs in the following format (see  Fig. \ref{fig:prompt}): (i) `lmk\_X' which is satisfied if the robot {is near} a landmark with a semantic label `lmk' and unique identifier $X$; (ii) `p\_lmk\_X' which is satisfied if the robot picks up the landmark with label `lmk' and unique identifier $X$; (iii) `pd' that is satisfied if the robot puts down a carried object (i.e., opens its gripper); and (iv) `photo' that is satisfied if the robot takes a photo. For instance, the AP `p\_box\_1' is satisfied if the robot picks up box 1. {We note that the identifier $X$ is omitted in (i)–(ii) when the landmark in the NL task is not associated with a unique identifier.} Also, observe that (i)–(ii) correspond to specific landmarks, whereas (iii)–(iv) do not. We emphasize that although we provide the LLM with the general structure of APs, we do not supply a predefined list of APs; instead, our algorithm infers them directly from the NL task.

{\textbf{Granularity of LLM's confidence:}}
As mentioned in Sec.~\ref{sec:primaryLLM}, we query the model $m$ times, which determines the granularity of the LLM’s confidence estimates for a given response (see the frequency score in \eqref{eq:freqM}). Smaller values of $m$ yield a less accurate ranking of responses by confidence, which can lead to higher help rates; however, larger values of $m$ increase the number of API calls and the associated costs. Thus, careful consideration is required when selecting $m$. {We emphasize again that the value of $m$ does not affect our translation success rate guarantees as stated in Thm. \ref{thm1}. This is because CP depends on the accuracy of ranking of the responses from lowest to highest model confidence, and not on the actual values of the NCS. Even if the scores fail to approximate this ranking, the marginal coverage still holds, albeit with excessively large prediction sets. In all our experiments, we select $m=10$ unless otherwise specified.}

\textbf{Eliminating Invalid APs:} 
\textcolor{black}{As soon as the LLM generates a response $s^{(j)}(k)$ at a step $k$ of Alg. \ref{alg:translation} (see Section \ref{sec:primaryLLM}), we first examine if that response is a logical/temporal operator. If not, we examine its candidacy as an AP, i.e., if it satisfies the rules mentioned above. APs that do not satisfy these rules are filtered out as discussed in Sec. \ref{sec:primaryLLM} (e.g., `pick\_box\_1' would be eliminated). Note that APs that do not have a numerical identifier are not filtered out (e.g., `p\_box').}

\textbf{Semantic Similarity:} Let $s^{i}(k),s^{j}(k)$ be two valid APs that were generated as responses at step $k$. Before checking their semantic similarity (see \eqref{eq:semantic_sim}), we first check if their numerical identifiers match. If not (e.g., `box\_1' and `package\_2'), we set their semantic similarity to zero as these APs refer to two different landmarks. 
Otherwise, we proceed with the computation of the cosine similarity of their textual parts (e.g., `box' and `package') as
\begin{align}\label{eq:semantic_sim}
    \cos(s^{(i)}(k),s^{(j)}(k))=\frac{v_i\cdot v_j}{\lVert v_i\rVert\lVert v_j\rVert},
\end{align} {where $v_i$ and $v_j$ are the vector representations of $s^{i}(k)$ and $s^{j}(k)$, respectively, computed using pre-trained LLM embeddings. The two APs are then merged if the cosine similarity exceeds a specified threshold $\zeta=0.75$. A higher threshold requires greater semantic similarity for merging to occur, influencing the granularity of the LLM's confidence indirectly. For instance, if the threshold is set too high, multiple semantically similar responses may remain unmerged and their frequencies will not be aggregated to reflect the LLM's overall confidence in those responses. {Conversely, thresholds that are too low may result in incorrectly treating semantically distinct responses as similar, which can also lead to a poor estimation of the model’s confidence.}
However, as discussed earlier, the granularity of the LLM confidence {estimates} does not affect correctness of Thm. \ref{thm1} but it may affect the help rates.}



\textbf{Evaluation Dataset \& Distribution $\ccalD$:} To evaluate our method, we generated a dataset of $1,000$ translation scenarios paired with their corresponding LTL formulas constructed as discussed in Remark \ref{rem:multiFeasFinal}. We classify the considered scenarios into three categories: \textit{easy}, \textit{medium}, and \textit{hard} depending on the number of APs needed to construct the ground truth formula. \textit{Easy} scenarios include 365 NL tasks that would require at most $2$ APs 
in the LTL formula. An example of easy tasks is, `Never pick up crate 8' \textcolor{black}{which can be expressed in LTL as: \textit{$\phi=\Box\neg$crate\_8}}.
\textit{Medium} scenarios include 440 NL tasks that would  require $3$ or $4$ APs. 
For instance, consider the NL task `Pick up bottle 3 from kitchen 2 but never enter living room 1' that can be expressed using the following LTL formula: \textit{$\phi=\Diamond$(kitchen\_2 $\wedge$ p\_bottle\_3)$\wedge \square \neg$living\_room\_1}. \textit{Hard} scenarios include 195 NL tasks that include more than 4 APs in the LTL formula. 
An example of such an NL task is `Avoid going to region 5 until you reach street 4 while delivering pass 4 from house 4 to office 1', which can be expressed using the LTL formula \textit{$\neg\text{region\_5}\ccalU\text{street\_4}\wedge\lozenge(\text{house\_4}\wedge\lozenge(\text{p\_pass\_4}\wedge\lozenge(\text{office\_1}\wedge\text{pd})))$}.\footnote{\textcolor{black}{The LTL formula corresponding to this NL task can also be defined as $\bar{\phi}=\neg \text{region\_5}\ccalU (\text{street\_4}\wedge\lozenge(\text{house\_4}\wedge\lozenge(\text{p\_pass\_4}\wedge\lozenge(\text{office\_1}\wedge \text{pd}))))$ depending on the interpretation of the 'while' clause in the NL task. Note that $\bar{\phi}$ is not semantically equivalent to $\phi$ and therefore does not fall under the case discussed in Remark \ref{rem:multiFeasFinal}. Such ambiguities can naturally arise in NL tasks. During the construction of our dataset, we selected $\phi$ as the ground truth formula, as it aligns with our intended task interpretation; consequently, $\bar{\phi}$ is treated as an incorrect option. {These subjective interpretations of ambiguous NL instructions are implicitly captured by the underlying distribution $\ccalD$. A different distribution $\bar{\ccalD}$ could instead assign $\bar{\phi}$ as the ground truth. If the user who provides the calibration sequences differs from the user at deployment time, such differences in how ambiguous instructions are interpreted can lead to distribution shifts that may compromise translation performance. We emphasize that our translation success guarantees hold with respect to the chosen distribution $\ccalD$.}}}
%
%
Similar to related works that employ CP \cite{ren2023robots,lindemann2023safe,wang2024safe}, we construct a \textit{distribution $\ccalD$} (see Section \ref{sec:problem}) that samples uniformly the scenarios from this dataset. {However, more complex distributions can be considered; we emphasize that CP and, consequently, our translation algorithm is agnostic to $\ccalD$.} 

\textbf{Evaluation Metrics:}
(i) Success rate: the percentage of translation scenarios that are successfully handled.
(ii) User help rate: the percentage of responses $s(k)$ that were generated using human help across all validation scenarios.
(iii) Help frequency across scenarios: the percentage of translation scenarios in which user help was required at least once during formula construction, denoted by $H_f$.

\begin{table}[t]
    \centering
    \caption{Empirical Translation Success Rates and Help Rates} \vspace{-0.3cm}
    \label{table:cp_comp}
    \begin{tabular}{ccccc}
        \toprule
        Framework & \makecell{$1-\alpha$\\ (\%)} & \makecell{Empirical \\ Success Rate (\%)} & \makecell{Help Rate \\ (\%)} & \makecell{$H_f$ \\ (\%)}\\
        \midrule
        & 95 & 94.44 & 0 & 0 \\
        \makecell{ConformalNL2LTL \\ w/ auxiliary LLM} & 97 & 96.77 & 0.279 & 2.5 \\
        & 99 & 99.24 & 0.378 & 4 \\
        \midrule
        & 95 & 95.33 & 3.41 & 30 \\
        \makecell{ConformalNL2LTL \\ w/o auxiliary LLM} & 97 & 97 & 4.03 & 36 \\
        & 99 & 98.2 & 7.43 & 36.5 \\
        \bottomrule
    \end{tabular}
    \vspace{-0.6cm}
\end{table}

\subsection{Empirical Evaluation in In-Distribution Settings}\label{sec:alphaVal}
In this section, we empirically validate that ConformalNL2LTL achieves user-defined translation success rates, as guaranteed by Th. \ref{thm1}. To demonstrate this,  we repeat the following process for $1-\alpha\in\{95\%, 97\%, 99\%\}$. We draw $220$ scenarios from $\ccalD$, where $200$ of them are used for calibration and the rest of them are used as test scenarios. For each test scenario, Alg. \ref{alg:translation} returns an LTL formula and then we manually check if it is the correct one. We compute the ratio of how many of the generated formulas are the correct ones. We repeat this  $10$ times. The average ratio across all experiments is the (empirical) translation success rate.

The results are summarized in the first row of Table \ref{table:cp_comp}. Observe that the translation rates are $94.44\%$, $96.77\%$, and $99.24\%$ for $1-\alpha=95\%,97\%$ and $99\%$ respectively, validating Th. \ref{thm1}. Minor deviations, are expected due to the finiteness of the calibration and validation sets and the number of trials; see also \cite{angelopoulos2021gentle}. The user help rates were recorded as $0\%$, $0.279\%$ and $0.378\%$ for $1-\alpha=95\%,97\%$ and $99\%$ respectively. As expected, higher values for $1-\alpha$ increase the size of the prediction sets and, consequently, the likelihood that ConformalNL2LTL requests user assistance.
The help frequency $H_f$, i.e., the fraction of translation scenarios requiring user help at least once, is $0\%$, $2.5\%$ and $4\%$, when $1-\alpha$ was $95\%$, $97\%$ and $99\%$ respectively. The increasing trend of $H_f$ is also expected as higher $1-\alpha$ values require more help. This monotonic increase is again expected, but note that both the help rates and $H_f$ remain sufficiently low. 

Most mistakes made by the LLMs stem from incorrect placement of parentheses. For example, for the following NL task `Go to monument 3, take a picture and stay there', the primary model generated the formula \textit{$\lozenge$(monument\_3 $\wedge$ photo)$\wedge\square$ monument\_3, instead of $\lozenge$(monument\_3 $\wedge$ photo$\wedge\square$monument\_3)}. \textcolor{black}{However, this problem can be bypassed by using prefix LTL formula, although it would require rewriting the prompts accordingly.} Another common error was the generation of nonsensical APs that still conformed to the AP generation rules.

\subsection{Effect of the Auxiliary LLM and Number $m$ of Queries}\label{sec:auxModelHelp} 

In what follows, we empirically demonstrate the help-rate benefits gained by incorporating an auxiliary model into our translation framework. 
We repeat the experiments of Section \ref{sec:alphaVal} using the exact same calibration and validation data, but with the auxiliary model entirely removed from both calibration and test-time operation. In this setting, during calibration, the NCS in \eqref{eq:NCS} is computed solely using the primary model $\psi_{\text{p}}$, and at test time the prediction set in \eqref{eq:predSetStep} simplifies to $\ccalC(\ell(k))=\ccalC(\ell(k),\psi_{\text{p}})$ for all $k$. As a consequence, whenever the primary model is uncertain—i.e., when $|\ccalC(\ell(k))|>1$—the algorithm must directly request help from a user. Following the same reasoning as in Section \ref{sec:theory}, this framework still achieves the target translation success rate $1-\alpha$. However, as expected, removing the auxiliary model may substantially increase both the user help rate and the help frequency $H_f$.

The performance of ConformalNL2LTL without an auxiliary model is reported in the second row of Table \ref{table:cp_comp}. When $1-\alpha$ is $95\%$, $97\%$, and $99\%$, the translation success rates are $95.33\%$, $97\%$, and $98.2\%$, respectively, which are consistent with Th. \ref{thm1}, confirming that the auxiliary model does not affect the theoretical success rate guarantees. In contrast, the corresponding user help rates rise to $3.41\%$, $4.03\%$, and $7.43\%$, which are significantly higher than those obtained when using an auxiliary model (see Section \ref{sec:alphaVal}). Likewise, $H_f$ increases significantly to $30\%$, $36\%$, and $36.54\%$ for $1-\alpha=95\%, 97\%, 99\%$, respectively. These empirical results highlight the primary role of the auxiliary LLM which is to reduce the uncertainty in selecting response $s(k)$. This in turn decreases both the overall user help rate and the fraction of formulas requiring user intervention.

In addition, to show the influence of $m$ on the help rate, we also ran experiments of ConformalNL2LTL without the auxiliary model for $1-\alpha=95\%$ and $m=5$ (recall that all previous experiments used $m=10$). We observed that the user help rate increased from $3.41\%$ to $4.71\%$ and that $H_f$ increased from $30\%$ to  $58\%$. The empirical translation rate is $95.15\%$ validating Thm. \ref{thm1}. The increase in help rates and $H_f$ is expected as lower $m$ results in a more inaccurate estimate of the LLM's confidence as discussed in Section \ref{sec:exp_nltask}.

\vspace{-0.4cm}
\textcolor{black}{\subsection{Comparative Evaluations against Translation Baselines}\label{sec:compareTranslate}}

We compare ConformalNL2LTL against three baselines: (A) an uncertainty-agnostic (UA) variant of ConformalNL2LTL; (B) Lang2LTL \cite{liu2023grounding}; and (C) nl2spec \cite{cosler2023nl2spec}; 
{Below, we provide an overview of these baselines followed by a discussion on how we configured them to ensure fair comparisons.}

(A) The UA version of ConformalNL2LTL, called UA-NL2LTL, completely removes CP, making decisions solely based on the mis-calibrated confidence scores of the primary LLM. Thus, at every step $k$, UA-NL2LTL picks the response $s(k)= \arg \max_{\bar s\in\Bar{\mathcal{S}}^{\psi_{\text{p}}}(k)} F(\bar s|\ell(k), m_k,\mathcal{S}^{\psi_{\text{p}}}(k))$. As a result, this baseline never asks for help, \textcolor{black}{neither from the auxiliary LLM, nor from the user}.

(B) Lang2LTL employs a modular translation pipeline consisting of: (a) an expression recognition module that extracts landmark-related substrings from the NL task using LLMs; (b) a grounding module that maps each entity to an AP; and (c) a lifted translation module that substitutes APs with placeholder symbols before generating the LTL formula using LLMs. 

(C) nl2spec uses components from `chain-of-thought' prompting to allow the LLM to reason about its response to raise its own accuracy. The framework first generates sub-translations by using the LLM which are then edited or removed by a user, who is also allowed to add subtranslations\footnote{We would like to note that ConformalNL2LTL can detect when help is required, while nl2spec does not.}. \textcolor{black}{A sub-translation provides context for proper formalization by providing translation examples for part of the NL task.} The LLM is then queried $M$ times using the subtranslations and the prompt, with the user providing help in the loop. \textcolor{black}{The translation with the highest confidence among the $M$ responses is displayed, along with alternative translations.}

To ensure fair comparisons, we have set up the baselines as follows:
(i) UA-NL2LTL is set up using the primary model of ConformalNL2LTL (i.e., GPT 4o). UA-NL2LTL and ConformalNL2LTL share exactly the same prompts to query the primary model.
(ii) We provide the translation rules in part (a) of our prompt (see Fig \ref{fig:prompt}) to Lang2LTL, and nl2spec such that all methods share the same AP recognition structure with ConformalNL2LTL. 
(iii) We also augmented the list of few-shot examples provided in Lang2LTL and nl2spec with ours. 
(iv) We do not allow the user to add, edit or remove sub-translations for nl2spec as this would lead to an unfair comparison with other baselines, which do not get help from the user. {Moreover, nl2spec lacks a mechanism to detect when assistance is needed and therefore requires user involvement in all translation scenarios in order to add, remove, or edit sub-translations. In contrast, ConformalNL2LTL explicitly reasons about when assistance is required to achieve a user-defined translation success rate. Additionally, the nature of user assistance differs fundamentally between the two approaches: nl2spec relies on user-specified sub-translations, whereas ConformalNL2LTL only requests the user to select among candidate responses at specific QA steps. Consequently, a direct comparison of help rates between nl2spec and ConformalNL2LTL would neither be fair nor meaningful}
We note that modifications (ii) and (iii) improved the performance of the baselines, as without them the {LLM frequently produced nonsensical APs}.

\begin{table}[t]
    \centering
    \caption{Translation Success Rates of Baselines}\vspace{-0.3cm}
    \label{table:baseline_comp}
    \begin{tabular}{cc}
        \toprule
         & Success Rates \\
        \midrule
        UA-NL2LTL & $87.4\%$ \\
        Lang2LTL w/ Few-Shot & $82.87\%$ \\
        Lang2LTL w/o Few-Shot & $53.53\%$ \\
        nl2spec & $73.16\%$ \\
        CMAS & $88\%$ \\
        \bottomrule
    \end{tabular}  
    \vspace{-0.6cm}
\end{table}

We report the translation success rate of these baselines in Table \ref{table:baseline_comp} over $182$ NL tasks, {sampled from the distribution $\ccalD$} introduced in Sec. \ref{sec:exp_nltask}, consisting of $71$ easy, $75$ medium and $37$ hard scenarios. The overall translation success rates were $87.4\%$ for UA-NL2LTL, $82.87\%$ for Lang2LTL with {our} few shot examples, $53.53\%$ for Lang2LTL without {our} few shot examples, and $73.16\%$ for nl2spec. Observe that UA-NL2LTL tends to outperform the baselines. However, the latter may be a case- and implementation-specific result. In addition, we observe that the performance of Lang2LTL dropped significantly when operating with its original few shot examples justifying the modification (iii) mentioned earlier. Finally, we emphasize that ConformalNL2LTL surpasses all these baselines in terms of translation success rates, as it can achieve user-specified $1-\alpha$ success levels while requiring minimal user intervention (see Table~\ref{table:cp_comp}).

%


\textcolor{black}{\subsection{Comparative Evaluations against NL-based Planners}\label{sec:comparePlanner}}

We compare ConformalNL2LTL against CMAS, an NL-based planner that maps NL instructions directly to robot plans without generating intermediate LTL formulas \cite{chen2023scalable}. Next, we discuss how CMAS and ConformalNL2LTL are configured.

{CMAS is developed for teams of $N$ robots, but we evaluate it in a single-robot setting ($N=1$). The output of CMAS is a robot plan defined as a sequence of actions generated to satisfy the input NL instructions. To ensure fair comparisons, we provide CMAS with (i) \textcolor{black}{an action space that shares the same structure as the AP recognition component of ConformalNL2LTL}, and
(ii) the same few-shot examples included in our prompt, i.e., the same NL tasks accompanied by the corresponding true plan.}  \textcolor{black}{Our definition of the action space in (i) ensures that the plans generated by CMAS are sequences of APs that need to be satisfied. As a result, both CMAS and ConformalNL2LTL reason at the same level of abstraction, avoiding unfair comparisons such as requiring CMAS to produce lower-level outputs (e.g., waypoints).}
%
{ConformalNL2LTL can be paired with any complete LTL planner to convert the generated LTL formulas into robot plans such as \cite{luo2021abstraction}. 

We then evaluate the plan success rate, defined as the proportion of plans successfully accomplish the specified NL task. 
Notice that the plan success rate of our method (ConformalNL2LTL \& LTL planner) is identical to the translation success rate of ConformalNL2LTL given that the employed planner, proposed in \cite{luo2021abstraction}, is correct and complete. The same applies to all translation baselines considered in Section \ref{sec:compareTranslate}.
The plan success rates for CMAS  was $88\%$ over the $182$ NL tasks used in Section \ref{sec:compareTranslate}. Note this success rate is comparable to the one that would have been achieved by integrating UA-NL2LTL with an LTL planner ($87.4\%$). Finally, we emphasize that, as in Section \ref{sec:compareTranslate}, our framework (ConformalNL2LTL \& LTL planner) outperforms CMAS as it can achieve user-defined $1-\alpha$ success rates with minimal human intervention.} Demonstrations of ConformalNL2LTL, integrated with the LTL planner proposed in \cite{luo2021abstraction}, in real-world navigation and mobile manipulation tasks are provided in \cite{video_demo}

\subsection{Empirical Evaluation in Out-Of-Distribution Settings}\label{sec:ood}
In this section, we evaluate the performance of ConformalNL2LTL in OOD translation scenarios. Specifically, we consider test scenarios generated from an unknown distribution $\hat{\ccalD}\neq \ccalD$ which is inaccessible and therefore cannot be used to construct calibration datasets. As a result, the prediction sets $\ccalC(\ell_{\text{test}}(k))$ in \eqref{eq:predSetStep} are constructed using calibration data sampled from $\ccalD$ while translation problems at test time are sampled from $\hat{\ccalD}$.


\textcolor{black}{\textbf{Distribution $\hat{\ccalD}$}:
The distribution $\hat{\ccalD}$ generates scenarios $\sigma=\{\xi,\ccalA\}$ associated with the following three types.}
(A) The first type refers to \textit{communication tasks} with action space $\ccalA_1=\{${send sensor data},{ relay a message},{ receive confirmation from agent i},{ trigger alert}$\}$. The  rules provided to the LLM for AP construction are: (i) `send$\_$data' which is true when the sensor data is sent, (ii) `relay$\_$msg' which is true when the system relays a message it received, (iii) `confirm$\_$i' that is true when the system gets confirmation from agent i, and (iv) `alert' that is true when an alert is set off. An example NL task in this category is {`Send sensor data as soon as you get confirmation from agent 4'.}
%
(B) The second type refers to \textit{human robot interaction tasks} with action space $\ccalA_2=\{${detect question},{ answer the question},{ display a greeting},{ receive hello},{ play a happy animation},{ play some audio}$\}$. The APs are required to have the following format: (i) `question' that is true when the robot detects a question from a user, (ii) `answer' which is true after the robot has answered a question, (iii) `display$\_$greeting' that is satisfied when the robot displays a greeting, (iv) `hello' that is true when the robot detects a person saying hello, (v) `happy$\_$animation' that is true when a happy animation is displayed, and (vi) `audio' which is true when playing the audio. A typical NL task in this category is `Display the happy animation, say hello and then answer the question provided by the user'.
(C) \textcolor{black}{The third type refers to \textit{exploration and information gathering tasks} with action space $\ccalA_3=\{${go to region X},{ explore},{ detect low map uncertainty}$\}$.  The APs are required to have the following format: (i) `region$\_$X' that is true when the robot reaches region X, (ii) `explore' that is true when switching to an exploration policy for expanding the map, (iii) `uncertainty$\_$low' which is true when the map uncertainty is low. An example of NL task in this category is `Go to region A, and switch to exploration mode if the map uncertainty is not low'.}

\textcolor{black}{We built a dataset of $200$ translation scenario–LTL pairs containing $73$ \textit{easy}, $88$ \textit{medium}, and $39$ \textit{hard} scenarioss using the APs and action spaces described above. The categorization into easy, medium, and hard follows the procedure in Section~\ref{sec:exp_nltask}. The distribution $\hat{\ccalD}$ is defined uniformly over this dataset. We emphasize again that neither $\ccalD$ nor $\hat{\ccalD}$ is known to our framework, and more complex distributions could also be considered.}

\begin{table}[t]
    \centering
    \caption{Empirical Translation Success Rates and Help Rates for OOD scenarios} \vspace{-0.3cm}
    \label{table:ood}
    \begin{tabular}{cccc}
        \toprule
        \makecell{$1-\alpha$\\ (\%)} & \makecell{Empirical \\ Success Rate (\%)} & \makecell{Help Rate \\ (\%)} & \makecell{$H_f$ \\ (\%)}\\
        \midrule
        95 & 79.5 & 0.37 & 3.5 \\
        97 & 80.5 & 0.40 & 4 \\
        99 & 84.5 & 0.67 & 6 \\
        \bottomrule
    \end{tabular}
    \vspace{-0.6cm}
\end{table}

In Table \ref{table:ood}, 
we report the performance of ConformalNL2LTL. We sample $20$ scenarios from $\hat{\ccalD}$. Alg. \ref{alg:translation} then generates the LTL specifications, which we manually check for accuracy. {We repeat the experiments conducted in Section \ref{sec:alphaVal} for the validation scenarios sampled from $\hat{\ccalD}$ {and calibration scenarios sampled from $\ccalD$}.}
Observe that for $1-\alpha=95\%,97\%,$ and $99\%$ the empirical translation success rates are $79.5\%,80.5\%,$ and $84.5\%$, respectively. This demonstrates that the translation success guarantees do not hold under distribution shifts; nevertheless, ConformalNL2LTL still performs reasonably well given the severity of the shift. The help rate and $H_f$ also continue to increase as $1-\alpha$ increases, consistent with trends observed earlier.

We also noticed that most translation errors occur mainly due to the absence of the ground truth response in either or both of the generated prediction sets $\ccalC(\ell(k),\psi_\text{p})$ and $\ccalC(\ell(k),\psi_{\text{aux}})$. This is expected as the probabilistic inclusion guarantees provided by CP no longer apply in OOD. A potential approach to enable CP to handle distribution shifts is to leverage robust CP to obtain valid prediction sets for all distributions $\hat{\ccalD}$ that are `close' to $\ccalD$ (based on the $f-$divergence) \cite{cauchois2024robust}. Integrating ConformalNL2LTL with robust CP is however out of the scope of this work and part of our future work.

%% file: conclusion.tex
This paper introduced ConformalNL2LTL, a novel NL-to-LTL translation framework. A key contribution of our approach is its ability to achieve user-defined translation success rates {while minimally asking for help from users.} Future work will focus on (i) relaxing the i.i.d. assumption required in Theorem \ref{thm1}; and (ii) formally integrating ConformalNL2LTL with existing multi-robot LTL planners. 



%% file: appendix_CP.tex
\appendices
\section{Background on Conformal Prediction} \label{appendix:CP}
Conformal prediction (CP) is a user-friendly paradigm for creating statistically rigorous uncertainty sets (called prediction sets) for the predictions of pre-trained machine learning models. 
These sets that are guaranteed to contain the ground-truth output with a user-specified probability without making assumptions on the underlying data distribution or the structure of the model.
In what follows, we provide an overview of CP for models used in classification settings as e.g., in
\cite{angelopoulos2020uncertainty,kumar2023conformal}; for further details see \cite{angelopoulos2021gentle}. 

Consider a pre-trained model $\psi$ that maps an input $x$ (e.g., an image or a natural-language question) to one of $K$ discrete labels. We assume that $\psi$ outputs a confidence vector $\psi(x)\in[0,1]^K$, where $\psi(x)_y$ denotes the predicted confidence assigned to a label $y$. Let $y^{\text{gt}}$ denote the ground-truth label associated with input $x$. We also assume that data pairs $(x,y)$ are drawn from an unknown distribution $\ccalD$.
Consider a calibration dataset $\{(x_i,y_i^{\text{gt}})\}_{i=1}^D\sim\mathcal{D}$ and a test input $x_{\text{test}}\sim\ccalD$ with unknown label $y_{\text{test}}^{\text{gt}}$. CP constructs a prediction set $\mathcal{C}(x_{\text{test}})$ such that
\begin{align}\label{equ:CP_guarantee}
  1-\alpha\leq P_{x_\text{test}\sim\ccalD}(y_\text{test}^{\text{gt}}\in\ccalC(x_\text{test}))\leq 1-\alpha+\frac{1}{D+1},  
\end{align}
where $\alpha\in(0,1)$ is a user-specified error level. To construct $\mathcal{C}(x_{\text{test}})$, CP relies on a nonconformity score (NCS), which heuristically quantifies how poorly the model predicts the true label. In classification problems, a common choice of NCS is $\Bar{r}=1-\psi(x_i)_{y^{\text{gt}}}$, which assigns higher scores to inputs $x$ for which the model places low confidence on the true class $y^{\text{gt}}$. The calibration procedure consists of two steps. First, we compute the NCS $\Bar{r}_i=1-\psi(x)_{y_i^{\text{gt}}}$ for each calibration pair $(x_i,y^{gt}_i)$. Second, we compute the $\frac{\lceil(D+1)(1-\alpha)\rceil}{D}$-quantile of the NCSs $\{\Bar{r}_i\}_{i=1}^D$ where $\lceil.\rceil$ is the ceiling function. We denote the resulting quantile by $\Bar{q}$. At test time, given an input $x_{\text{test}}$, we construct the following prediction set $\ccalC(x_\text{test})=\{y:\psi(x_\text{test})_y\geq 1-\Bar{q}\}$ which satisfies \eqref{equ:CP_guarantee}. 

\section{{Conformal Prediction with {Multiple Semantically Equivalent} Formulas}}\label{app:proof}

The CP analysis, provided in Section \ref{sec:cp}, as well as our translation success guarantees (Theorem \ref{thm1}) can be extended to cases where {multiple semantically equivalent formulas} exist. As discussed in Remark \ref{rem:multiFeasFinal}, the key difference lies in the construction of the calibration dataset introduced in Section \ref{sec:cp} where among all LTL formulas that are semantically equivalent to the NL instructions, we select the one  constructed by picking at each step the {response} with the highest frequency. 
%
%
%
CP can then be applied as usual, yielding prediction sets (see \eqref{eq:maxFreqGuarantee}) that contain, with user-specified probability, the semantically equivalent LTL formula (among all other equivalent formulas) constructed using the highest-frequency responses, among all correct responses, at each translation step. Then, {as discussed in Remark \ref{rem:multiFeasFinal},} we emphasize that the guarantees provided in Proposition \ref{prop:new} and Theorem \ref{thm1} hold with respect to that LTL formula. In what follows, we provide a more formal description of how CP is applied to construct prediction sets when multiple semantically equivalent formulas exist, which is adapted from Appendices A3-A4 of \cite{ren2023robots}. 

\textbf{Constructing a Calibration Dataset:} We sample $D\geq 1$ independent calibration scenarios 
$\{\sigma_{i,\text{calib}}\}_{i=1}^D\sim\mathcal{D}$. Each $\sigma_{i,\text{calib}}$ gives rise to a sequence of $H_i \geq 1$ prompts denoted by 
\begin{align}\label{eq:seqProm}
&\bar{\ell}_{i,\text{calib}}=\ell_{i,\text{calib}}(1),\dots,\ell_{i,\text{calib}}(H_i),
\end{align}
where each prompt $\ell_{i,\text{calib}}(k)$ is constructed exactly as in Section \ref{sec:primaryLLM}; the only difference is that part (d) contains partial formulas constructed using all possible correct responses given $\ell_{i,\text{calib}}(k-1)$, for $k\in\{2,\dots,H_i\}$. {Let  $R_i^k\geq 1$ denote the number of correct responses at step $k$. Hereafter, we denote by $\bar{\ccalT}_{i,\text{calib}}(k)=\{s_{i,r,\text{calib}}(k)\}_{r=1}^{R_i^k}$ the set of all feasible/correct responses at step $k$.} Thus, each sequence $\bar{\ell}_{i,\text{calib}}$ is associated with a set of $R_i\geq 1$ {semantically equivalent} LTL formulas $\bar{\Phi}_{i,\text{calib}} = \{\phi_{i,r,\text{calib}} \}_{r=1}^{R_i}$, where
\begin{align}\label{eq:Plans3}
&\phi_{i,r,\text{calib}}=s_{i,r,\text{calib}}(1),\dots,s_{i,r,\text{calib}}(H_i)
\end{align}
representing a syntactically correct LTL formula that is semantically equivalent to the NL task. This gives rise to a calibration dataset $\mathcal{M} = \{(\bar{\ell}_{i,\text{calib}}, \bar{\Phi}_{i,\text{calib}}) \}_{i=1}^D$. 
{As in Section \ref{sec:cp}, this process results in converting the distribution $\ccalD$ into an equivalent distribution $\bar{\ccalD}'$ over prompt sequences discussed as mentioned above.} 

{Given a validation scenario $\sigma_{\text{test}}\sim\ccalD$,} we construct the prompt sequence:
\begin{align}\label{eq:testbarell}
    \bar{\ell}_{\text{test}}=\ell_{\text{test}}(1),\ldots,\ell_{\text{test}}(H_{\text{test}}),
\end{align} where each $\ell_{\text{test}}(k)$ has the same structure as the prompt presented in Section \ref{sec:primaryLLM}, with part (d) containing the partial LTL formulas constructed by using all possible correct responses obtained up until $k-1$. 
Note that $\{\bar{\ell}_{i,\text{calib}}\}_{i=1}^D$ and $\bar{\ell}_{\text{test}}$ are i.i.d, sampled independently from $\bar{\ccalD}'$. 


\textbf{Extracting the Highest-Frequency LTL Formula:}
Consider a function $G$ that is applied to every point in the support of $\bar{\ccalD}'$. This function, given the prompt $\bar{\ell}(k)$, which is part of the sequence $\bar{\ell}\sim\bar{\ccalD}'$, and a set $\bar{\ccalT}(k)$ of feasible responses, returns the response $s\in\bar{\ccalT}(k)$
with the highest frequency $F^{\psi_{\text{p}}}(s|\ccalS(k))$. If $\bar{\ccalT}(k)=\emptyset$, then the user types in the response and $F^{\psi_{\text{p}}}(s|\ccalS(k))=0$. Applying $G$ auto-regressively to all calibration sequences $\bar{\ell}_{i,\text{calib}}$ defined in \eqref{eq:seqProm} yields the following sequence sequence of $H_{i,\text{calib}}$ prompts:
%
%
\begin{align}\label{eq:maxPrompt}
\bar{\ell}_{i,\text{calib}}^{\text{max}}=\ell_{i,\text{calib}}^{\text{max}}(1),\ldots,\ell_{i,\text{calib}}^{\text{max}}(H_i)
\end{align}where each prompt $\ell_{i,\text{calib}}^{\max}(k)$ is the same as  $\ell_{i,\text{calib}}$  with part (d)  containing the formula constructed using the responses {$s_{i,\text{calib}}^{\text{max}}(n)=G(\ell_{i,\text{calib}}^{\text{max}}(n),\bar{\ccalT}_{i,\text{calib}}(n)),\forall n\in\{1,\ldots,k-1\}$. }
Note that the prompt $\ell_{i,\text{calib}}^{\text{max}}(1)$ does not contain any response in part (d). 

The corresponding LTL formula is given by,
\begin{align}\label{eq:maxLTL}
\phi_{i,\text{calib}}^{\text{max}}=s_{i,\text{calib}}^{\text{max}}(1),\ldots,s_{i,\text{calib}}^{\text{max}}(H_i),
\end{align}where $s_{i,\text{calib}}^{\text{max}}(k)=G(\ell_{i,\text{calib}}^{\text{max}}(k),\bar{\ccalT}_{i,\text{calib}}(k))$.
{Observe that} ${\phi}^{\text{max}}_{i,\text{calib}}(k)$ is the LTL formula constructed by always selecting the response with the highest frequency among all correct options, at all steps until $k$. 
This process gives rise to the following calibration dataset $\ccalM''=\{({\ell}^{\text{max}}_{i,\text{calib}},{\phi}^{\text{max}}_{i,\text{calib}})\}_{i=1}^D$.

{\textbf{Constructing Prediction Sets:}} 
Consider a validation scenario $\sigma_{\text{test}}\sim\ccalD$ with unknown corresponding LTL formula. Given a set of calibration scenarios $\{\sigma_{i,\text{calib}}\}_{i=1}^D\sim\ccalD$, we first construct the calibration dataset $\ccalM=\{\bar{\ell}_{i,\text{calib}},\bar{\Phi}_{i,\text{calib}}\}$, where, as discussed earlier, $\bar{\ell}_{i,\text{calib}}\sim\bar{\ccalD}'$. 
From $\ccalM$ we generate the calibration set $\ccalM''=\{({\ell}^{\text{max}}_{i,\text{calib}},{\phi}^{\text{max}}_{i,\text{calib}})\}_{i=1}^D$ by applying $G$ auto-regressively, as described earlier. 
Similarly, we generate the sequence of prompts $\bar{\ell}_{\text{test}}^{\text{max}}=\ell_{\text{test}}^{\text{max}}(1),\ldots,\ell_{\text{test}}^{\text{max}}(H_{\text{test}})$ by applying $G$ auto-regressively to the sequence $\bar{\ell}_{\text{test}}$, defined in \eqref{eq:testbarell}, associated with $\sigma_{\text{test}}$. Then, we construct the following prediction set 
\begin{align}\label{eq:predsetmax}
    \bar{\ccalC}(\bar{\ell}_{\text{test}}^{\text{max}})=\{\phi|\bar{F}(\phi|\bar{\ell}_{\text{test}}^{\text{max}})\geq 1-\bar{q}\},
\end{align} 
where $\bar{F}(\phi|\bar{\ell}_{\text{test}}^{\text{max}})$ is as given in \eqref{eq:barF}, and $\bar{q}$ is the $(1-\alpha)$ quantile of the scores computed for the dataset $\ccalM''$. 

If the validation sequence $\bar{\ell}_{\text{test}}^{\text{max}}$ and the calibration sequences  ${\bar{\ell}}^{\text{max}}_{i,\text{calib}}$ are i.i.d. then the following coverage guarantee holds for $ \bar{\ccalC}(\bar{\ell}_{\text{test}}^{\text{max}})$ (that is `equivalent' to \eqref{eq:CP1})
\begin{align}\label{eq:maxFreqGuarantee}
P({\phi}^{\text{max}}_{\text{test}}\in\bar{\ccalC}(\bar{\ell}^{\text{max}}_{\text{test}}))\geq1-\alpha.
\end{align}
which follows directly from CP. \textcolor{black}{Here, $\phi_{\text{test}}^{\text{max}}$ is the formula constructed using the highest-frequency correct response, among all correct responses in $\ccalS(k)$, at each translation step $k$ (as in \eqref{eq:maxLTL})}. To verify the i.i.d. assumption, recall that both these test and calibration sequences are obtained by applying the measurable function $G$ to every point in the support of $\bar{D}'$, including  $\bar{\ell}_{i,\text{calib}}$ and $\bar{\ell}_{\text{test}}$. Since measurable functions of i.i.d. random variables remain i.i.d., it follows that $\bar{\ell}_{\text{test}}^{\text{max}}$ and $\{\bar{\ell}^{\text{max}}_{i,\text{calib}}\}_{i=1}^D$ are i.i.d. and therefore \eqref{eq:maxFreqGuarantee} holds.

{\textbf{On-the-fly Construction of Prediction Sets:}} As in Section \ref{sec:cp}, the prediction set in \eqref{eq:predsetmax} is constructed after the entire sequence $\bar{\ell}_{\text{test}}^{\text{max}}$ is obtained. However, at test time, the full sequence of prompts is not observed at once. Instead, we can follow exactly the same steps as in Section \ref{sec:cp} to construct the prediction set in \eqref{eq:predsetmax} causally. Consequently, \eqref{eq:predSetStep}–\eqref{eq:causalPred3}, together with all theoretical results in Section~\ref{sec:theory}, hold for the sequence $\bar{\ell}_{\text{test}}^{\text{max}}$ and the corresponding LTL formula $\phi_{\text{test}}^{\text{max}}$. \textcolor{black}{As a result, as discussed in Remark~\ref{rem:multiFeasFinal}, Theorem~\ref{thm1} holds when the user selects the correct response from the prediction set with the highest frequency $F^{\psi_{\text{p}}}(s \mid \ccalS(k))$. This requires ConformalNL2LTL to return both the prediction set and the associated frequencies.}